\documentclass{article} 

\usepackage[utf8]{inputenc} 
\usepackage[T1]{fontenc}    
\usepackage{hyperref}       
\usepackage{url}            
\usepackage{booktabs}       
\usepackage{amsfonts, amsmath, amsthm, mathtools}       
\usepackage{nicefrac}       
\usepackage{microtype}      
\usepackage{xcolor}         
\usepackage{wrapfig}

\usepackage[accepted]{icml2020}

\allowdisplaybreaks

\newtheorem{theorem}{Theorem}
\newtheorem{lemma}{Lemma}
\newtheorem{assumption}{Assumption}
\newtheorem{proposition}{Proposition}
\newtheorem{corollary}{Corollary}

\newcommand{\lft}{\mathrm{left}}
\newcommand{\rgt}{\mathrm{right}}
\newcommand{\idx}{\text{index}}
\newcommand{\rad}{\text{rad}}
\newcommand{\dist}{\text{dist}}

\DeclarePairedDelimiter\ceil{\lceil}{\rceil}
\DeclarePairedDelimiter\floor{\lfloor}{\rfloor}

\icmltitlerunning{Coarse-Grained Smoothness for RL in Metric Spaces}

\begin{document}

\twocolumn[

\icmltitle{Coarse-Grained Smoothness for RL in Metric Spaces}

\begin{icmlauthorlist}
\icmlauthor{Omer Gottesman}{brown}
\icmlauthor{Kavosh Asadi}{amazon}
\icmlauthor{Cameron Allen}{brown}
\icmlauthor{Sam Lobel}{brown}
\icmlauthor{George Konidaris}{brown}
\icmlauthor{Michael Littman}{brown}
\end{icmlauthorlist}

\icmlaffiliation{brown}{Brown University}
\icmlaffiliation{amazon}{Amazon Web Services}

\icmlcorrespondingauthor{Omer Gottesman}{omer\_gottesman@brown.edu}

\vskip 0.3in
]

\printAffiliationsAndNotice{}

\begin{abstract}

Principled decision-making in continuous state--action spaces is impossible without some assumptions. A common approach is to assume Lipschitz continuity of the Q-function. We show that, unfortunately, this property fails to hold in many typical domains. We propose a new coarse-grained smoothness definition that generalizes the notion of Lipschitz continuity, is more widely applicable, and allows us to compute significantly tighter bounds on Q-functions, leading to improved learning. We provide a theoretical analysis of our new smoothness definition, and discuss its implications and impact on control and exploration in continuous domains.

\end{abstract}

\section{Introduction}

Reinforcement learning (RL) aims to develop algorithms that learn optimal policies for sequential decision-making problems~\citep{sutton2018reinforcement}.  In RL, the dynamics of the environment are unknown, and therefore an agent must learn  using environmental interactions. While RL has been studied extensively when the set of environmental states is discrete, continuous domains pose a significant challenge:  impressive empirical results have been achieved thanks to advances in deep RL, but the theory underlying RL in continuous environments is still poorly understood. A recent body of work addresses this problem by assuming that the Q-function of the optimal policy is Lipschitz continuous \citep{ni2019learning, tang2020off, touati2020zooming}.  

The Lipschitz assumption is intuitive and allows strong theoretical results to be derived for a variety of algorithms, but we will show that it is, in fact, overly strong and limiting---and rarely satisfied in domains of interest. In particular, we argue that the key assumption about the smoothness of the Q-function is less about the Q-function itself and more about the ability of the metric to reflect the distance between states under the transition dynamics of the environment. Furthermore, rather than make assumptions on the structure of the Q-function, which are hard to confirm, our assumptions are over the dynamics of the domain, making them much easier to verify empirically or through domain knowledge. We then show that, by choosing a generalized smoothness definition that is qualitatively similar to Lipschitz continuity but designed to properly account for the topology of practical domains, theoretical guarantees similar to those provided by the Lipschitz continuity assumption can be obtained for a much wider range of domains. 

\section{Background and Notation}

We use the standard RL formulation of an MDP denoted by $\langle \mathcal{S}, \mathcal{A}, T, R, P_0, \gamma \rangle$, where $\mathcal{S}$, $\mathcal{A}$ and $\gamma$ are the state space, action space, and the discount factor, respectively. The next-state transition probabilities and reward functions are given by $T(s' | s, a)$ and $R(s, a)$, respectively. The initial state distribution is $P_0(s)$. We denote the state--action space as $\mathcal{X} = \mathcal{S} \times \mathcal{A}$ and a point in that space as $x = (s, a)$. The dimensionality of $\mathcal{X}$ is $D$.

A policy $\pi(a | s)$ gives the probability of taking an action at a given state. The value function is the expected return collected by following the policy from state $s$: $V^\pi(s) \coloneqq \mathrm{E} [ \sum_{t=0}^T \gamma^t r_t | a_t \sim \pi, s_0 = s]$. The function $Q^{\pi}(s, a)$ is the expected return for taking action $a$ at state $s$, and afterwards following $\pi$ in selecting future actions. An optimal policy, $\pi^*$, is defined as a policy that maximizes $V(s)$ for all $s$ and its value and Q-functions are denoted $V^*(s)$ and $Q^*(s, a)$, respectively.

We assume $\mathcal{X}$ and $\mathcal{S}$ are metric spaces and denote their distance functions as $d_{\mathcal{X}}(x, x')$ and $d_{\mathcal{S}}(s, s')$, respectively, but will drop the subscript when there is no ambiguity. Note that, in the more intuitive case where metrics are defined over $\mathcal{S}$ and $\mathcal{A}$, a metric over $\mathcal{X}$ can be easily defined as $d_{\mathcal{X}}(x, x') = d_{\mathcal{S}}(s, s') + C d_{\mathcal{A}}(a, a')$, for any $C \geq 0$. We leave further discussion regarding assumptions on the metric functions to later sections, as one of the main objectives of this paper is to establish what types of metrics are useful in an RL context.

\subsection{Lipschitz Continuity and Bounds}


The Lipschitz constant of a function $f$ is defined as 
\begin{align}
\label{eq:lip_constant_definition}
    L \equiv \sup \left\{ \frac{|f(x) - f(x')|}{d(x, x')}  \ \ \middle| \ \ x, x' \in \mathcal{X} \right\}.
\end{align}

If $L$ is finite, the function is said to be $L$-Lipschitz continuous. Intuitively, the Lipschitz constant can be thought of as the magnitude of the largest gradient of a continuous function, $f(x)$, over $\mathcal{X}$, and is infinite for discontinuous functions. If $L$ and the value of $f(x')$ at a set of points $x' \in \mathcal{X}$ are known, an upper and a lower bound for $f(x)$ can be obtained by
\begin{align}
\label{eq:lip_est}
    \hat{f}_{UB}(x) = \min_{x' \in \mathcal{X}} f(x') + d(x, x') L  \\ \nonumber
    \hat{f}_{LB}(x) = \max_{x' \in \mathcal{X}} f(x') - d(x, x') L,
\end{align}
as shown schematically in Fig.~\ref{fig:intuition} (a).

In the RL literature, these bounds are used to make generalizations about the Q-function that support reasoning about which areas of the state--action space should still be explored.  

\section{Related Work}

Numerous papers have employed Lipschitz assumptions to obtain sample-complexity results for exploration. \citet{pazis_and_parr} introduced a PAC-learning algorithm in metric spaces that maintains a set of known state--action pairs and uses an approximate nearest neighbor approach. Similarly, \citet{Q_learning_nearest} introduced the nearest-neighbor Q-learning algorithm that can output a near-optimal approximation of the Q-function in finite time. \citet{regret_bounds_for_undiscounted} tackle the undiscounted episodic RL problem and introduced an algorithm whose regret scales sub-linearly with the number of episodes. This algorithm was later improved~\citep{improved_regret_bounds} and simplified~\citep{efficient_model_free_in_metric_spaces}. \citet{osband2014model} leveraged Lipschitz continuity of the value function to obtain regret bounds that depend on the dimensionality, rather than cardinality, of the underlying MDP.

A more recent line of work focuses on learning adaptive discretization techniques for exploration in metric spaces~\citep{adaptive_discretization_sinclair,touati2020zooming}. These algorithms leverage Lipchitz continuity assumptions to provide upper bounds on the value function of a candidate set of state--action pairs chosen more frequently from promising areas, while maintaining optimism in the face of uncertainty, and are inspired by zooming approaches from the bandit literature~\citep{zooming_kleinberg,zooming_akshay}.

Lipschitz assumptions have been used to tackle numerous other challenges in RL. For example, \citet{lipschitz_pg} use Lipschitz assumptions to provide a new policy-gradient algorithm that can adaptively choose a learning rate that yields monotonic improvements in the policy-gradient objective. \citet{asadi2018lipschitz} show that having access to Lipschitz models is a key ingredient for combating the compounding-error problem in model-based reinforcement learning. \citet{safe_model_based} give a Lyapnunov approach, and \citet{chandak2020towards} propose a Seldonian algorithm to maintaining safety in RL that leverages Lipschitz assumptions on the model. \citet{lipschitz_non_staionary} propose a worst-case approach to addressing non-stationary problems where the rate of non-stationarity with respect to time satisfies a Lipschitz assumption. \citet{tang2020off} introduce Lipschitz Value Iteration, an algorithm that gradually tightens the range of Q-value estimates to perform off-policy evaluation. \citet{deep_mdp} lean on Lipschitz assumptions to build a representation and state abstraction that is conducive to value-function optimization, and \citet{omer_combining} use similar Lipschitz assumptions for combining different estimators for off-policy evaluation and applying the resultant estimator to medical applications. Lipschitz continuity has also been utilized in the context of lifelong RL~\citep{lecarpentier2020lipschitz}.

Closely related to our work, alternative assumptions exist in the literature that, akin to Lipschitz continuity, capture the intuitive notion of smoothness in alternative ways. For example, \citet{metric_e3} make a local-modeling assumption---even more stringent than the Lipschitz assumption---that they use to extend the $E^3$ algorithm \citep{e3} to metric spaces. Finally, \citet{touati2020zooming} allow for a generalized variant of the Lipschitz assumption where an additive error term may exist, and show that their exploration algorithm is robust to such mis-specifications. A conceptual advantage of our generalized smoothness assumption compared to \citet{touati2020zooming}, is that our approach is explicit in defining a length-scale over which changes of the Q-function are ignored, allowing for a principled way of specifying exactly how the Lipschitz smoothness is violated, as well as quantifying the smoothness of the function over multiple length-scales.

More generally, Lipschitz continuity has garnered attention in other areas of machine learning and in particular in the deep learning literature. Most notably, Lipschitz properties have been used to provide norm-based capacity control~\citep{neyshabur2015norm}, to learn generative models for matching distributions in high-dimensional spaces~\citep{arjovsky2017wasserstein}, and for combating adversarial examples and improving the robustness of neural networks against malevolent actors~\citep{finlay2019improved}.

\section{Motivating Example: Q-function Discontinuities in Continuous Riverswim}
\label{sec:intuition}

To provide intuition for why alternative smoothness measures to Lipschitz continuity should be explored, we examine the smoothness properties of the Q-function in a simple environment with characteristics we believe are common to many continuous control tasks of interest; we formalize these characteristics in Sec.~\ref{sec:q_func_structure}. In many domains of interest, such as navigation and robotics, the reward function is derived from attempting to reach a set of goal states, and is thus discontinuous. We will assume the metric function, $d_{\mathcal{S}}(s, s')$, quantifies how quickly an agent can transition between states---we will later argue that this assumption, although strong, can replace the smoothness assumption on the Q-function, and is much easier to verify or enforce at problem-formulation time.

The particular environment we use for illustration is a continuous version of the riverswim environment \citep{strehl2008analysis}. The 1D state is a real number, the initial state is $s_0 = 0$, and the continuous actions are $a \in [-a_{\max}, a_{\max}]$. The transition function is $s_{t+1} = s_t + a_t - c$, where $0 \leq c \leq a_{\max}$. The dynamics can be seen as modeling swimming against a current. An episode ends with a small positive reward, $r_{\lft}$, upon reaching $s \leq -1$, and a larger reward, $r_{\rgt} > r_{\lft}$, upon reaching $s \geq 1$. This is a good example of an environment in which an agent must explore effectively, or it will never discover the large reward at $s \geq 1$.

Here, the transition function is continuous both with respect to the state and actions. However, the goal-based reward function results in a discontinuous Q-function. The optimal policy is to always take action $a=a_{\max}$.\footnote{If $r_{\lft} > r_{\rgt} \gamma^{\floor{\frac{2}{a_{\max} - c}}}$ the optimal policy will have states from which it will move left, but for simplicity we ignore that region of the parameter space.} The value function and Q-function for the optimal policy is
\begin{figure}[t]
\vskip 0.2in
\begin{center}
\includegraphics[width=0.48\columnwidth]{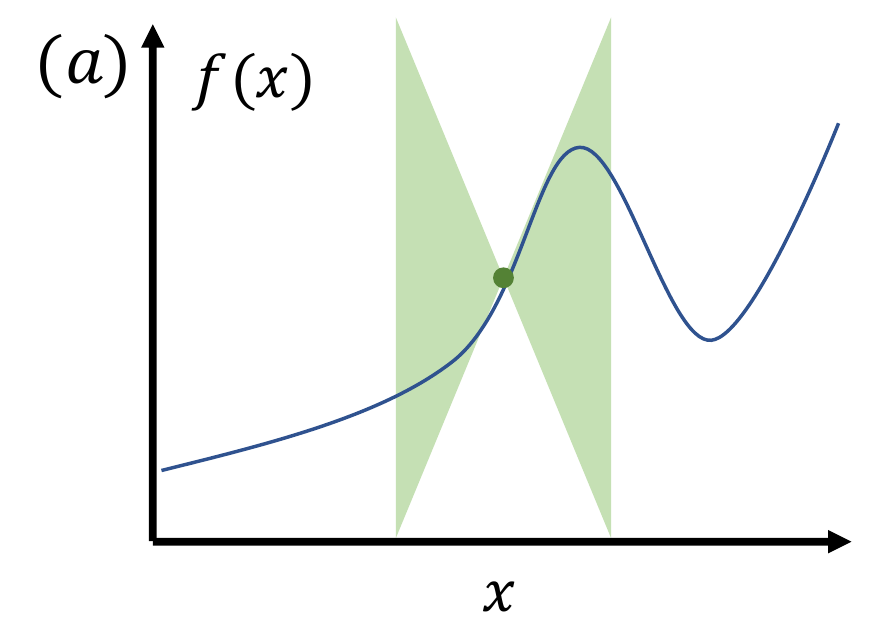}
\includegraphics[width=0.48\columnwidth]{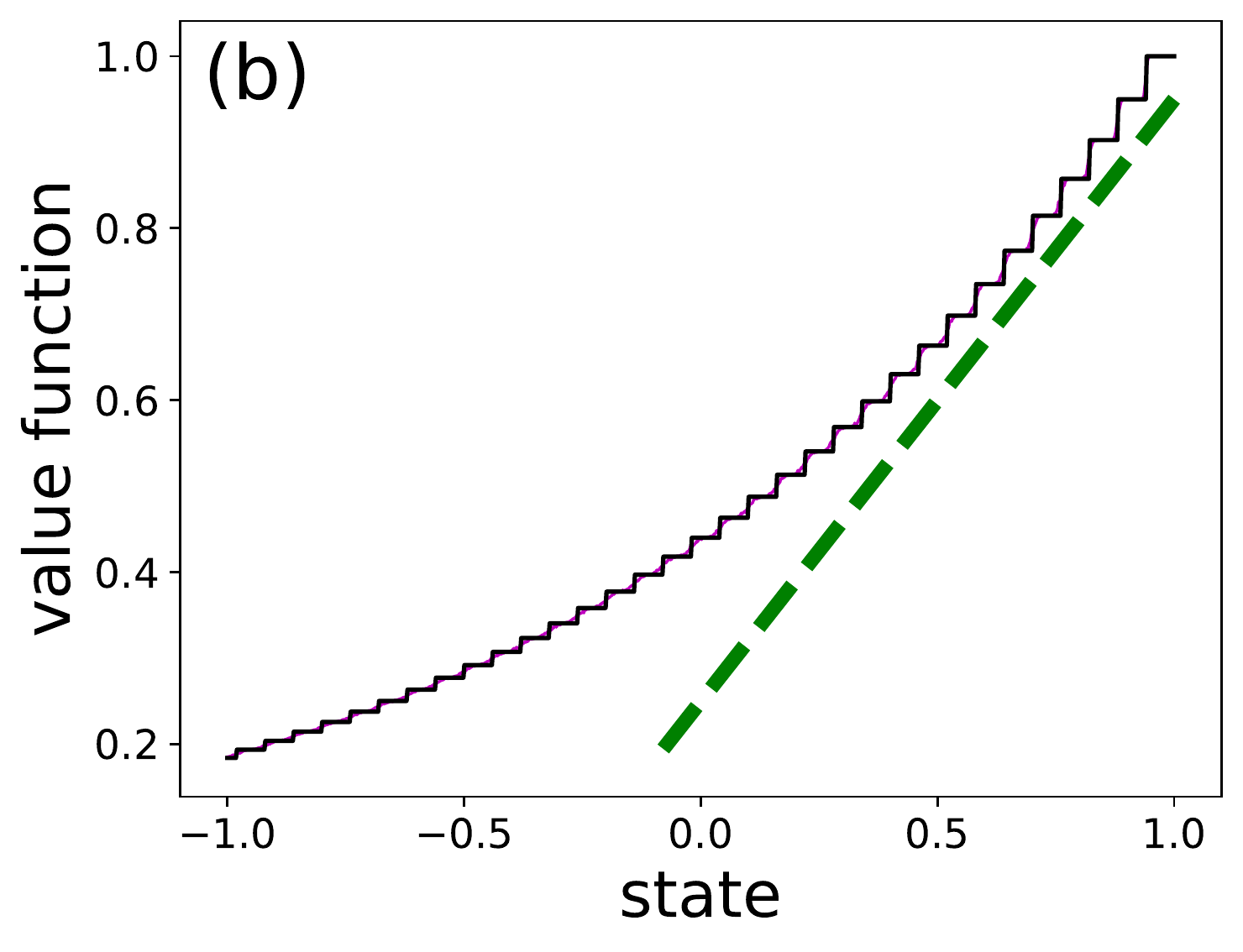}
\includegraphics[width=0.48\columnwidth]{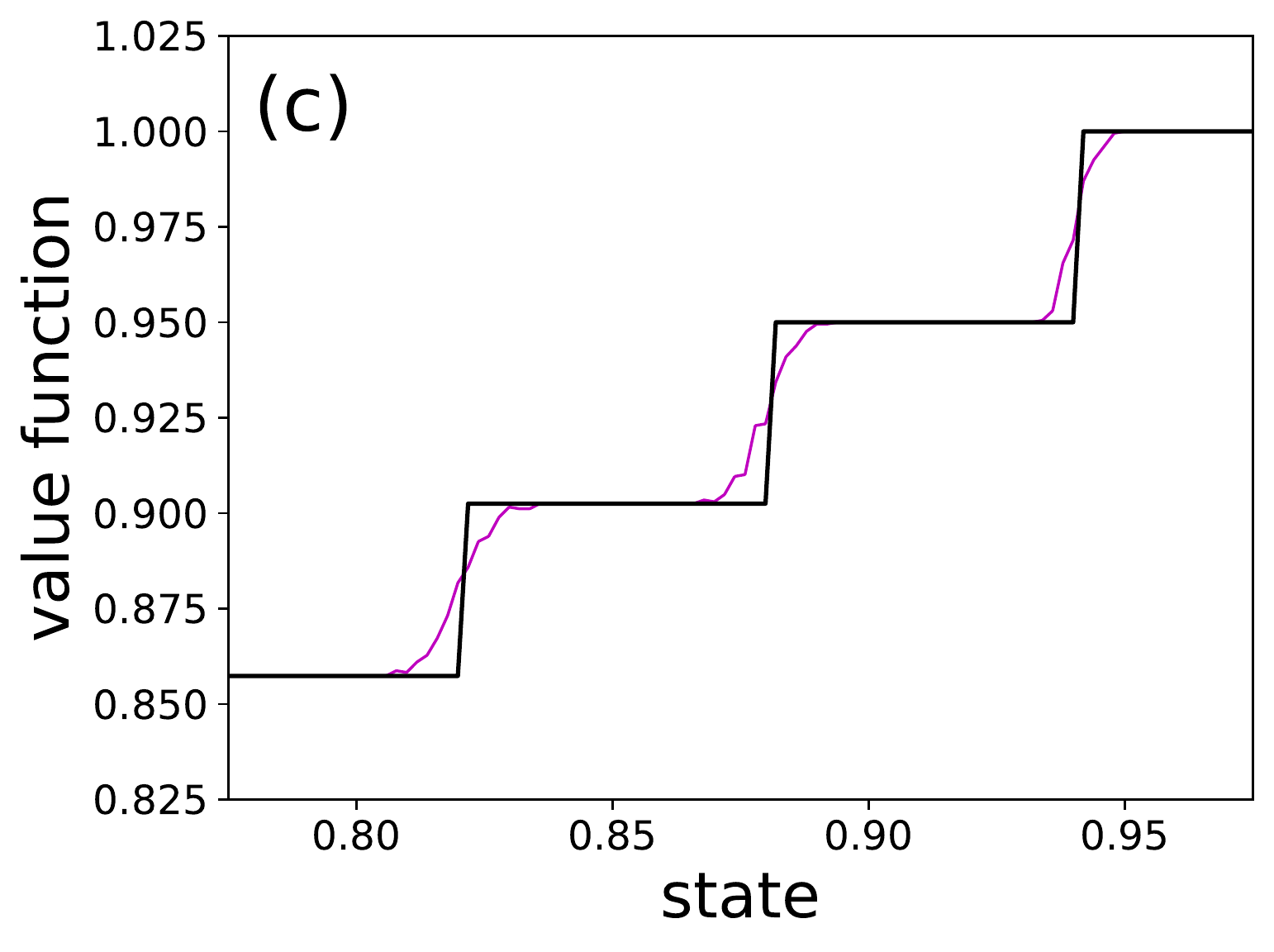}
\includegraphics[width=0.48\columnwidth]{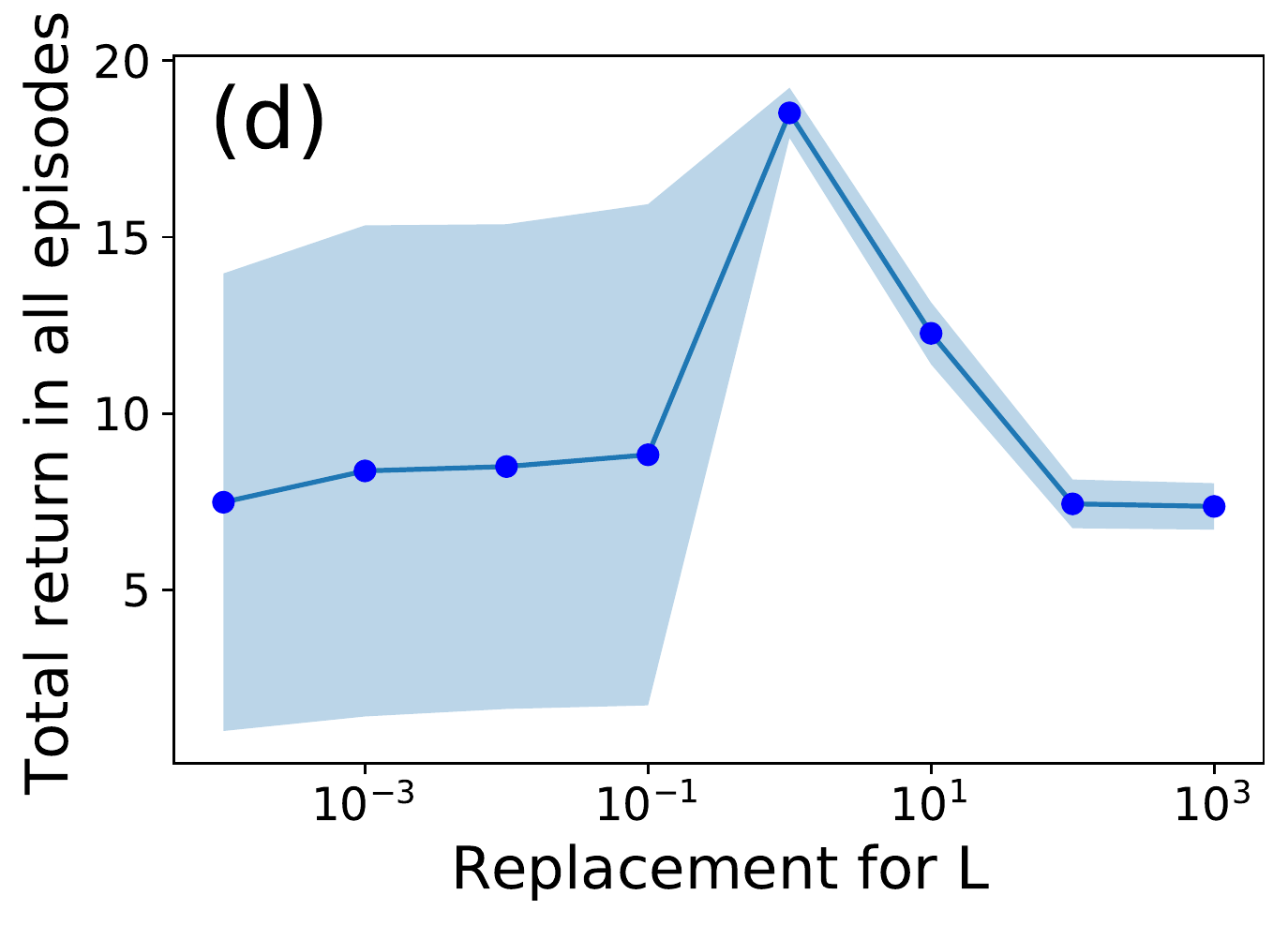}
\caption{(a) Schematic of Lipschitz bounds. (b) The value function for the optimal policy in the continuous riverswim environment for both deterministic (black) and stochastic (magenta) dynamics. The dashed green line represents a coarse grained slope of the function. (c) Magnification of the previous plot. (d) Performance of the algorithm presented in \citet{ni2019learning} for different replacement values for $L$.}
\label{fig:intuition}
\end{center}
\vskip -0.2in
\end{figure}

\begin{align}
    V^*(s) &= r_{\rgt} \gamma^{\lceil \frac{1 - s}{a_{\max} - c} \rceil}, \\
    Q^*(s, a) &= \begin{cases} 
      r_{\lft} & s + a \leq c - 1  \\
      \gamma V^*(s + a - c) & c - 1 < s + a < c + 1 \\
      r_{\rgt} & c + 1 \leq s + a
   \end{cases},
\end{align}
where $\lceil \cdot \rceil$ denotes rounding up to the closest integer. Note that there is a discontinuous jump in $Q^*$ at every state where the number of steps required to reach the goal increases by one, as well as when an action will lead to such a state. Thus, $Q^*$ has discontinuities with respect to both states and actions and clearly violates the Lipschitz property. In Sec.~\ref{sec:q_func_structure}, we argue that this phenomenon is common to a wide variety of domains.

In Fig.~\ref{fig:intuition} (b), we plot $V^*$ in black and see that, despite the discontinuous structure of the value function, in any region in the state space, a relatively clear slope can be seen if one ignores the discontinuous steps---qualitatively approximated by the green line. Furthermore, we also plot (in magenta) the empirically estimated state values when Gaussian noise with standard deviation of $0.03 a_{\max}$ is added to the transition (Fig.~\ref{fig:intuition} (c) zooms in on the upper right part of Fig.~\ref{fig:intuition} (b) for clarity). Although adding stochasticity eliminates the discontinuity of the Q-function, it is still the case that its maximal gradient is much larger than its coarse-grained rate of change.

The Lipschitz constant is usually used to bound the Q-function in one region of the state--action space based on observations in other regions. Therefore, it seems intuitive that it would be better to use the effective slope rather than the true ``infinite'' Lipschitz constant. In Sec.~\ref{sec:altenative_smoothness_measures}, we formalize this idea.

To illustrate why the coarse-grained slope might be a better measure of the smoothness of the Q-function, we apply the control algorithm proposed by \citet{ni2019learning} to the riverswim environment with different inputs as the Lipschitz-constant value, and plot the total reward obtained over 50 episodes as a function of this value. Despite the true Lipschitz constant being ``infinite'', we see that using large replacement for $L$ leads to poor performance due to over-exploration. The algorithm performs best for an intermediate replacement for $L \approx 1$, which tellingly matches the largest coarse-grained slope shown by the green line in Fig.~\ref{fig:intuition} (d).\footnote{Although the Lipschitz constant of the Q-function is defined over the state--action space, here we approximate it by the slope in the $a=a_{\max}$ subspace.} Note also that using too small of a replacement for $L$ leads to poor exploration as seen by the low mean performance and large variance for these values. The results demonstrate that the Lipschitz constant is not necessarily the smoothness measure we should be using in RL algorithms, and therefore alternative measures should be explored.

\section{Constraints on the Geometry of Q-functions}
\label{sec:q_func_structure}

\begin{figure}{}
\begin{center}
\includegraphics[width=0.8\columnwidth]{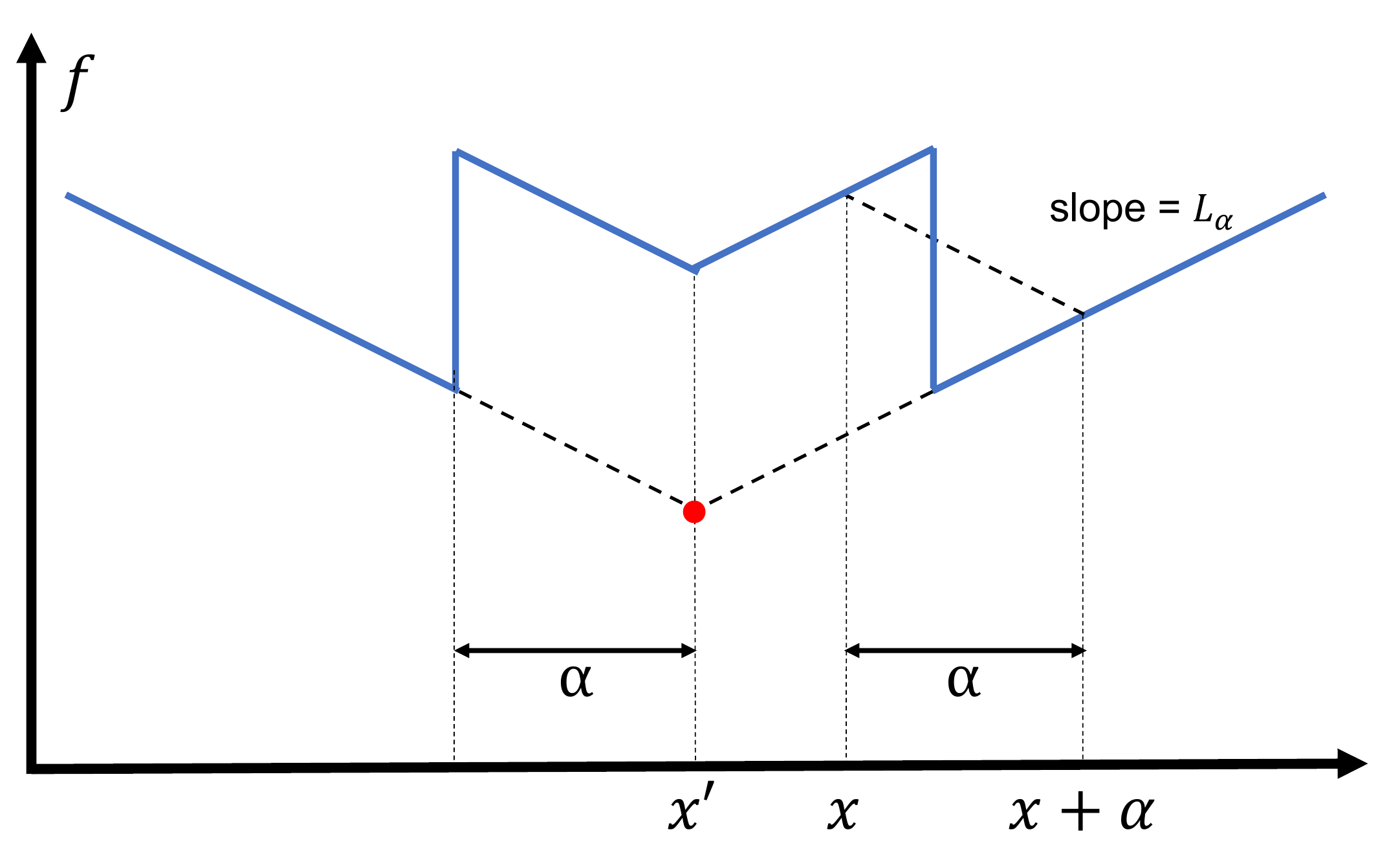}
\caption{Schematic demonstrating how the upper bound of a function is derived based on knowledge of its value at a point $x$.}
\label{fig:l_alpha_bounds_derivation}
\end{center}
\end{figure}

Motivated by the observations in the previous section, we now show that discontinuous Q-functions will be the norm rather than the exception, as they arise under conditions very common in RL. Therefore, care must be taken when designing algorithms that assume Lipschitz continuity. On the other hand, we show that these discontinuities can be bounded by making some assumptions regarding the metric over $\mathcal{S}$, which motivates the introduction, in Sec.\ref{sec:altenative_smoothness_measures}, of an alternative smoothness measure that can handle these discontinuities.

The key assumption we make formalizes the intuitive notion that the metric over the state space quantifies how quickly an agent can transition between states.

\begin{assumption}
\label{ass:metric_structure_1}
For any two states $s, s' \in \mathcal{S}$ such that $d(s, s') \leq d_{\min}$, we have that (i)  with probability $1-\delta$, $s'$ can be reached from $s$ in at most $k$ time-steps, and (ii) with probability 1, the agent can remain within a distance of less than $d_{\min}$ from $s'$.
\end{assumption}

This assumption can be applied to deterministic continuous domains (in which case $\delta = 0$) or stochastic discrete domains when a suitable metric is provided (or stochastic continuous domains under discretization). In Appendix \ref{appendix:assumption_modification}, we analyze how Assumption~\ref{ass:metric_structure_1} can be modified for stochastic continuous domains, and how to relax part (ii) of the assumption. Note that such a metric is not always trivial to obtain---consider for example a continuous 2D maze. Two points on opposing sides of a wall may be very close in Euclidean distance, but it may take the agent many steps to move from one to the other. However, such assumptions, or ones similar in nature, are arguably required for any RL approach in continuous domains where learning requires generalization to unobserved states. Furthermore, such assumptions are easier to validate or design into the problem formulation than assumptions on the structure of the Q-function. While it is often difficult to know what the $k$-step reachable neighborhood for an MDP is, it is often possible to do so for the special case of $k=1$, allowing for estimation of a $d_{\min}$ value for which Assumption \ref{ass:metric_structure_1} holds.

For deterministic continuous domains, Assumption \ref{ass:metric_structure_1} leads immediately to discontinuities, as we saw in Sec.~\ref{sec:intuition}. The simplest example demonstrating this property is to consider the boundary of a region from which a terminal goal state can be reached in one step (such a boundary arises by taking $k=1$ in Assumption~\ref{ass:metric_structure_1}). The value of all states within the boundary will be the reward, $r$, for that goal state, while states at an infinitely small distance outside that boundary will have value at most $\gamma r$. While adding stochasticity to the dynamics might remove the strict discontinuities, as long as the transition noise is much smaller than the average step sizes, local gradients, while finite, may still be very large, as can be seen in Fig.~\ref{fig:intuition} (c).

While the Q-function may have discontinuous ``jumps'' (infinite gradients), Assumption~\ref{ass:metric_structure_1} also implies that the \emph{magnitude} of such ``jumps'' can be bounded.

\begin{theorem}
\label{thm:max_jump_in_mdp}
Under Assumption~\ref{ass:metric_structure_1},
for any two points $x_1, x_2 \in \mathcal{X}$ such that $d(s_1, s_2) \leq d_{\min}$, the difference $\Delta Q^* \equiv |Q^*(x_1) - Q^*(x_2)|$ is bounded by
\begin{align}
    \Delta Q^* \leq \frac{1 - \gamma^k}{1 - \delta \gamma^k} \left( Q^*_{{\max},(1,2)} - \frac{r_{\min}}{1 - \gamma} \right),
\end{align}
where $Q^*_{{\max},(1,2)} \equiv \max(Q^*(x_1), Q^*(x_2))$.
\end{theorem}

The proof can be found in Appendix \ref{appendix:proof_max_jump_in_mdp}, but the main idea is to note that, for any two states less than $d_{\min}$ from each other, the agent can transition from the lower value state to the higher value state, and then follow the policy from the higher value state. Thus, all that is needed to complete the proof is to compute the rewards and discount factor accumulated along the way between the two states.

The bound in Theorem~\ref{thm:max_jump_in_mdp} is a local one in the sense that it depends on the Q-values of the state--action points compared, and reflects the fact that, if the values of the points are small, the difference between their values will be comparatively small as well. (For example, note that the ``steps'' in Fig.~\ref{fig:intuition} (b) are smaller where the value function is lower.) However, by substituting $Q^*_{\max}$, the maximum value for states in the domain\footnote{$Q^*_{\max}$ is usually known and depends on the dynamics of the MDP. For example, it is $r_{\max}$ in domains where the episode ends when a goal is reached or $r_{\max}/(1 - \gamma)$ in infinite horizon domains.}, for $Q^*_{{\max},(1,2)}$, a corollary of Theorem~\ref{thm:max_jump_in_mdp} provides a looser bound for the difference between any two points in $\mathcal{X}$ that we can obtain without knowledge of their Q-values.
\begin{corollary}
\label{thm:max_jump_in_mdp_global}
For any two points $x_1, x_2 \in \mathcal{X}$ such that $d(s_1, s_2) \leq d_{\min}$, the difference $\Delta Q^* \equiv |Q^*(x_1) - Q^*(x_2)|$ is bounded by
\begin{align}
    \Delta Q^* \leq \Delta Q^*_{\max} \equiv \frac{1 - \gamma^k}{1 - \delta \gamma^k} \left( Q^*_{\max} - \frac{r_{\min}}{1 - \gamma} \right),
\end{align}
where $Q^*_{\max}$ is the maximal value of any state in the environment.
\end{corollary}


We have demonstrated that, in contrast to the key assumption commonly made in the continuous RL literature, Q-functions are discontinuous for many domains of interest. However, the additional constraints we have proved regarding the geometry of the Q-function allow us to propose a new, more meaningful smoothness measure for use in RL, which we do next.

\section{Coarse-Grained Smoothness Measures}
\label{sec:altenative_smoothness_measures}

The results presented in Sec.~\ref{sec:q_func_structure} suggest that, while the Q-function may be discontinuous locally, there are constraints on how much it can change over a given distance in $\mathcal{X}$. This insight leads us to propose an alternative smoothness measure that ignores discontinuities on the local scale, defined by the length parameter $\alpha$, with a corresponding coarse-grained smoothness parameter, $L_{\alpha}$:
\begin{align}
\label{eq:lip_alpha}
    L_{\alpha} \equiv \sup \left\{ \frac{|Q^*(x) - Q^*(x')|}{d(x, x')} \ \ \middle| \ \ d(x, x') \geq \alpha \right\}.
\end{align}
Note that, unlike $L$, for any bounded function over a bounded metric space, $L_{\alpha}$ exists and is finite for all $\alpha>0$. Furthermore, $L_{\alpha}$ is a monotonically decreasing function of $\alpha$, and $\lim_{\alpha \rightarrow 0} L_{\alpha}=L$ is the standard Lipschitz constant.


We can use $L_{\alpha}$ for upper and lower bounds of a function in a manner similar to the use of the traditional Lipschitz constant (Eq.~\ref{eq:lip_est}). The upper bound that the function at point $x'$ imposes on the function's value at point $x$ is simply $f(x') + L_{\alpha} d(x, x')$ for $d(x, x') > \alpha$. 

For $d' \equiv d(x, x') \leq \alpha$, we can bound $f(x)$ in the following way (Fig.~\ref{fig:l_alpha_bounds_derivation} schematically demonstrates how these bounds are derived): By noting that $f(x' + d' + \alpha) \leq f(x') + L_{\alpha}(d' + \alpha)$, then because $d(x' + d' + \alpha, x) = \alpha$, we get $f(x) \leq f(x' + d' + \alpha) + L_{\alpha} \alpha \leq f(x') + L_{\alpha}(d' + \alpha) + L_{\alpha} \alpha = f(x') + L_{\alpha}(d' + 2 \alpha)$. Applying the same logic to the lower bound we get:
\begin{align}
\label{eq:lip_alpha_est}
    \hat{f}_{UB}^{\alpha}(x) = \min_{x'} f(x') + g_x^{\alpha}(x') \\ \nonumber
    \hat{f}_{LB}^{\alpha}(x) = \max_{x'} f(x') - g_x^{\alpha}(x') ,
\end{align}
\begin{align}
\label{eq:lip_alpha_est_g}
    g_x^{\alpha}(x') \equiv \begin{cases} 
      L_{\alpha} (d(x, x') + 2 \alpha) & d(x, x') \leq \alpha \\
      L_{\alpha} d(x, x') & d(x, x') > \alpha.
   \end{cases}
\end{align}

The derivation of these bounds is presented in scalar (1D) notation but is easily applied to any normed metric over $\mathbb{R}^n$, by applying the 1D derivation along the unit vector pointing from $x$ to $x'$.

An important point to note is that, while this derivation provides bounds on points at any distance from each other, the $L_{\alpha}$ smoothness property makes an assumption \emph{only} on points that are at least $\alpha$ away from each other. These bounds could be replaced by simpler more intuitive bounds, such as $g_x^{\alpha}(x') = \alpha L_{\alpha}$ for $d(x, x') \leq \alpha$, but such bounds would require additional assumptions on the geometry of the Q-function at small length-scales, which is exactly what our formulation aims to avoid.

Also of note is that the $L_{\alpha}$ bounds are discontinuous, which would make them more difficult to work with using gradient-based algorithms. However, most Lipschitz-based methods are non-parametric, and therefore do not use a parameterized representation of the bound, and for such algorithms the discontinuity of the bounds will not pose a problem.

Taking the limit $d(x, x') \rightarrow 0$ in Eq.~\ref{eq:lip_alpha_est} and \ref{eq:lip_alpha_est_g} leads to the following proposition bounding the maximum discontinuity gap in a $L_{\alpha}$-smooth function:

\begin{proposition}
\label{thm:max_jump_given_l_alpha}
For an $L_{\alpha}$-smooth function $f$, for any two points $x$ and $x'$,
\begin{align}
    \lim_{d(x, x') \rightarrow 0} |f(x) - f(x')| < 2 L_{\alpha} \alpha.
\end{align}
\end{proposition}


\begin{figure}[t]
\vskip 0.2in
\begin{center}
\includegraphics[width=0.32\columnwidth]{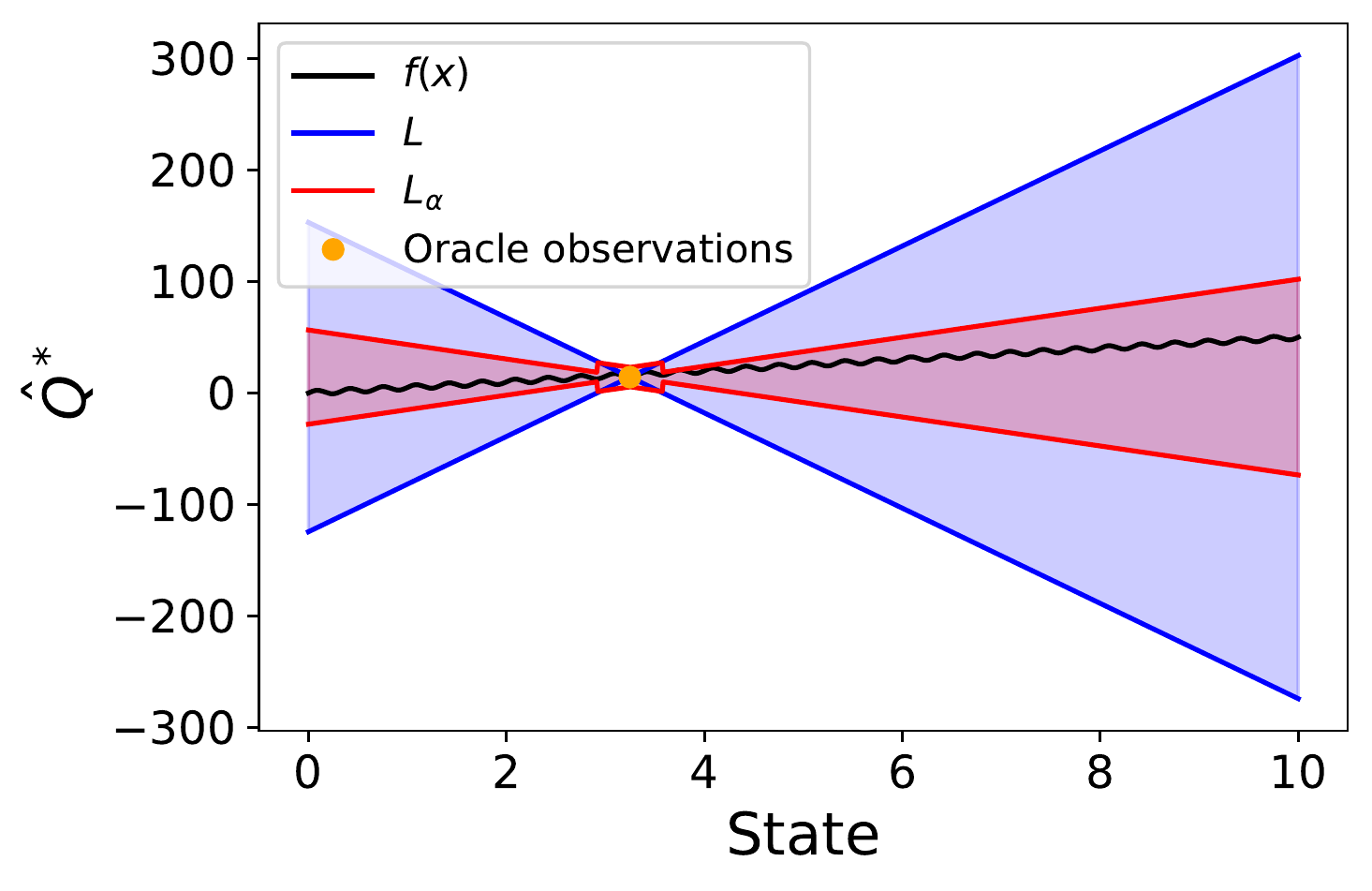}
\includegraphics[width=0.32\columnwidth]{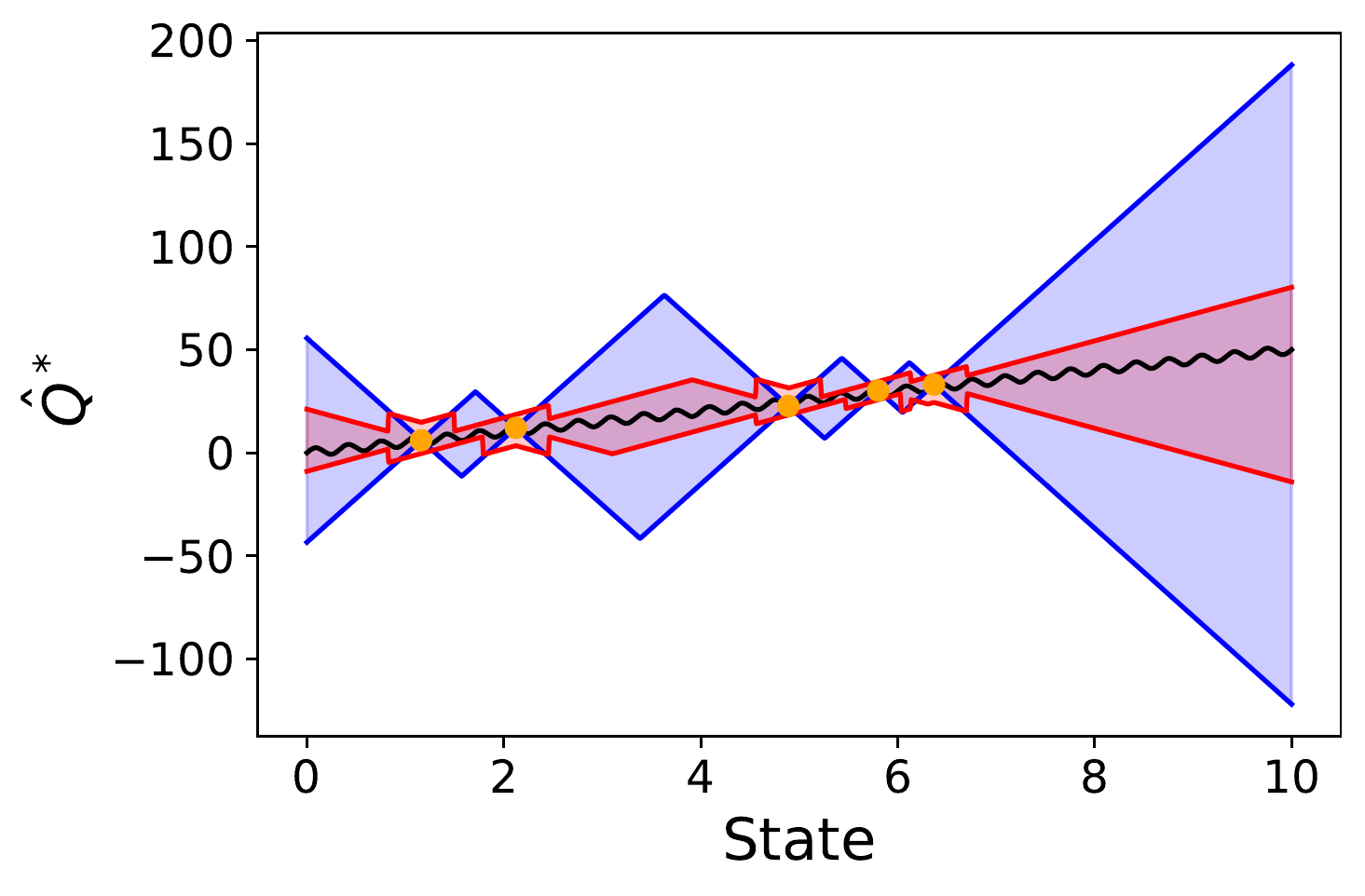}
\includegraphics[width=0.31\columnwidth]{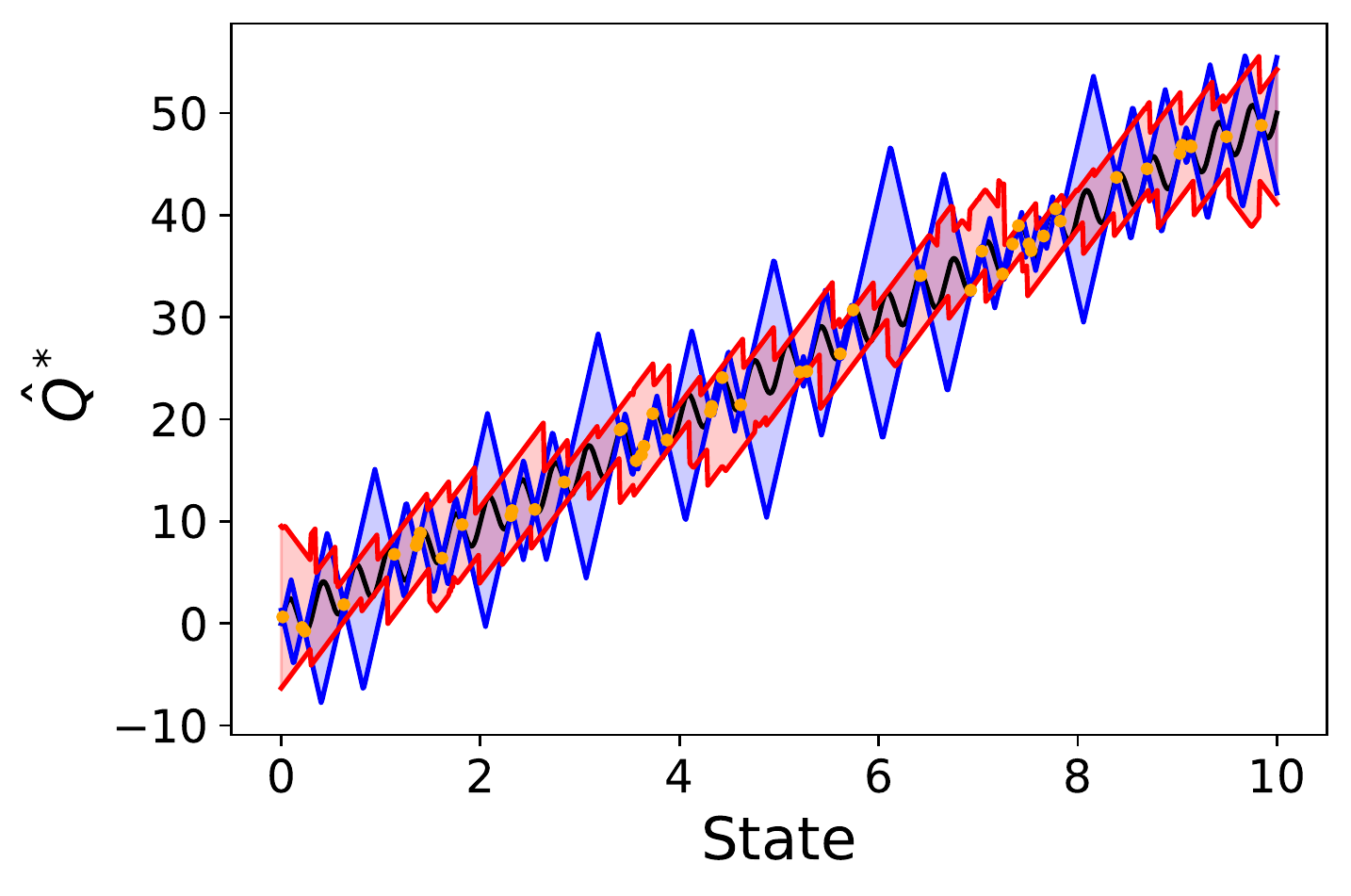}
\includegraphics[width=0.32\columnwidth]{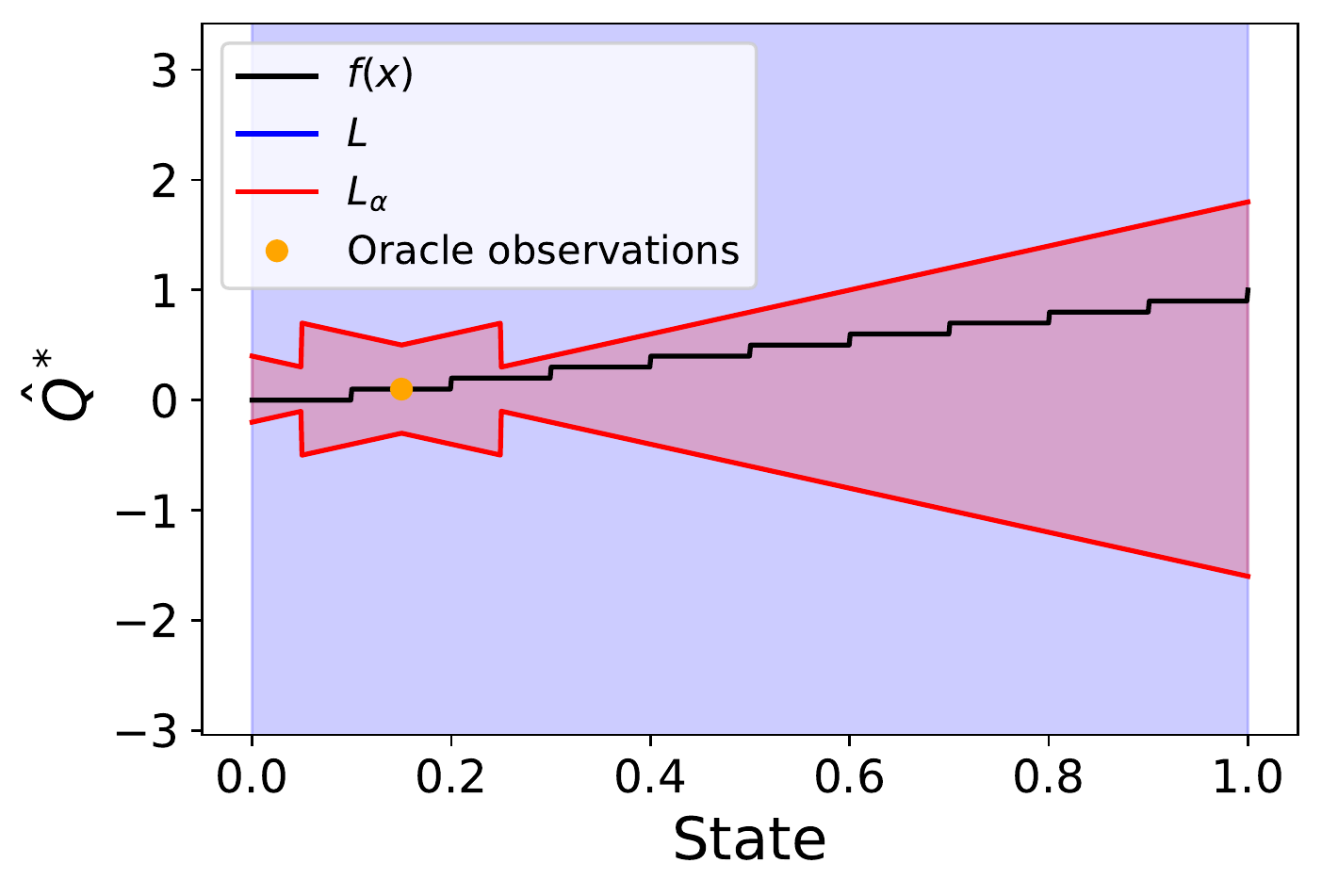}
\includegraphics[width=0.32\columnwidth]{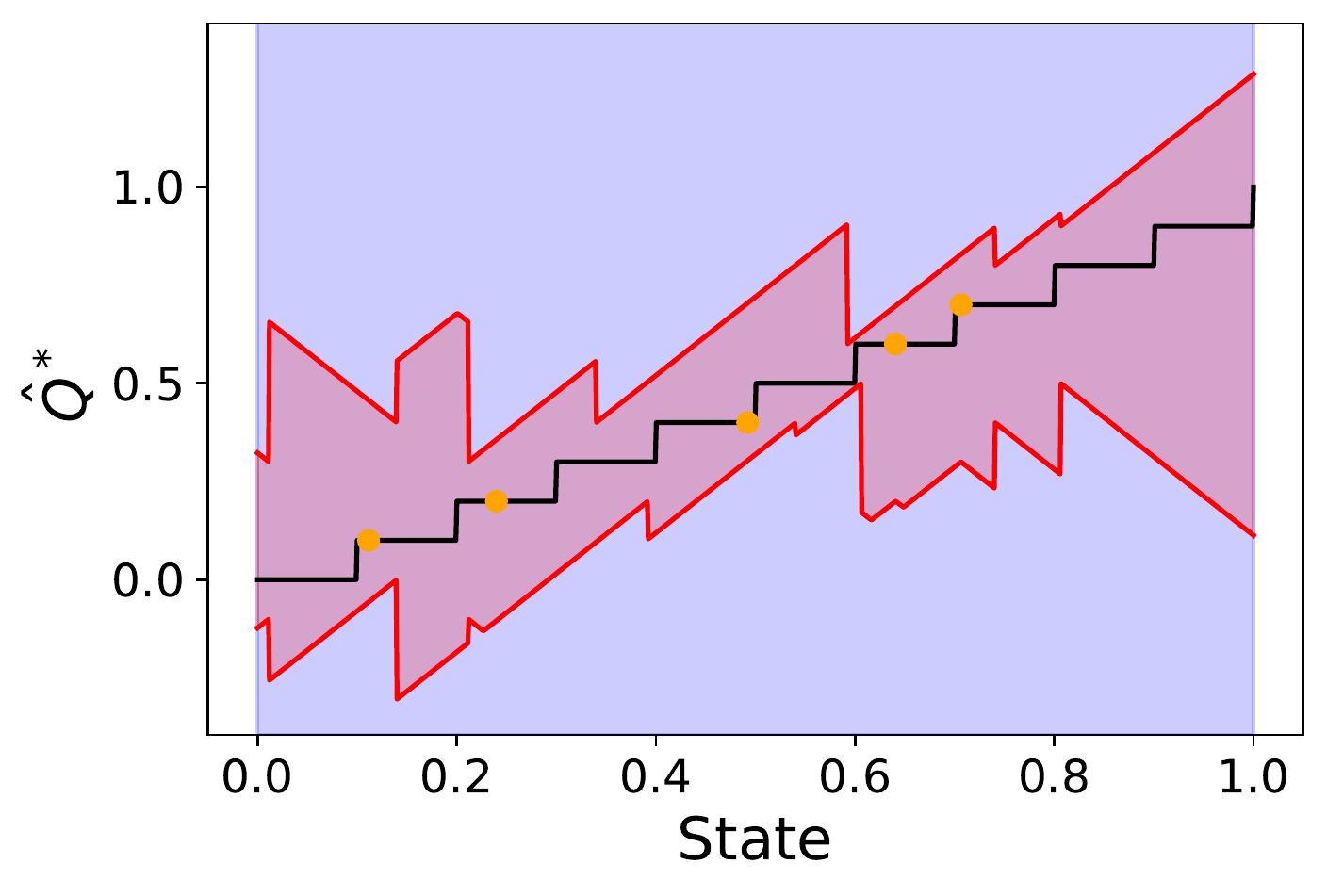}
\includegraphics[width=0.32\columnwidth]{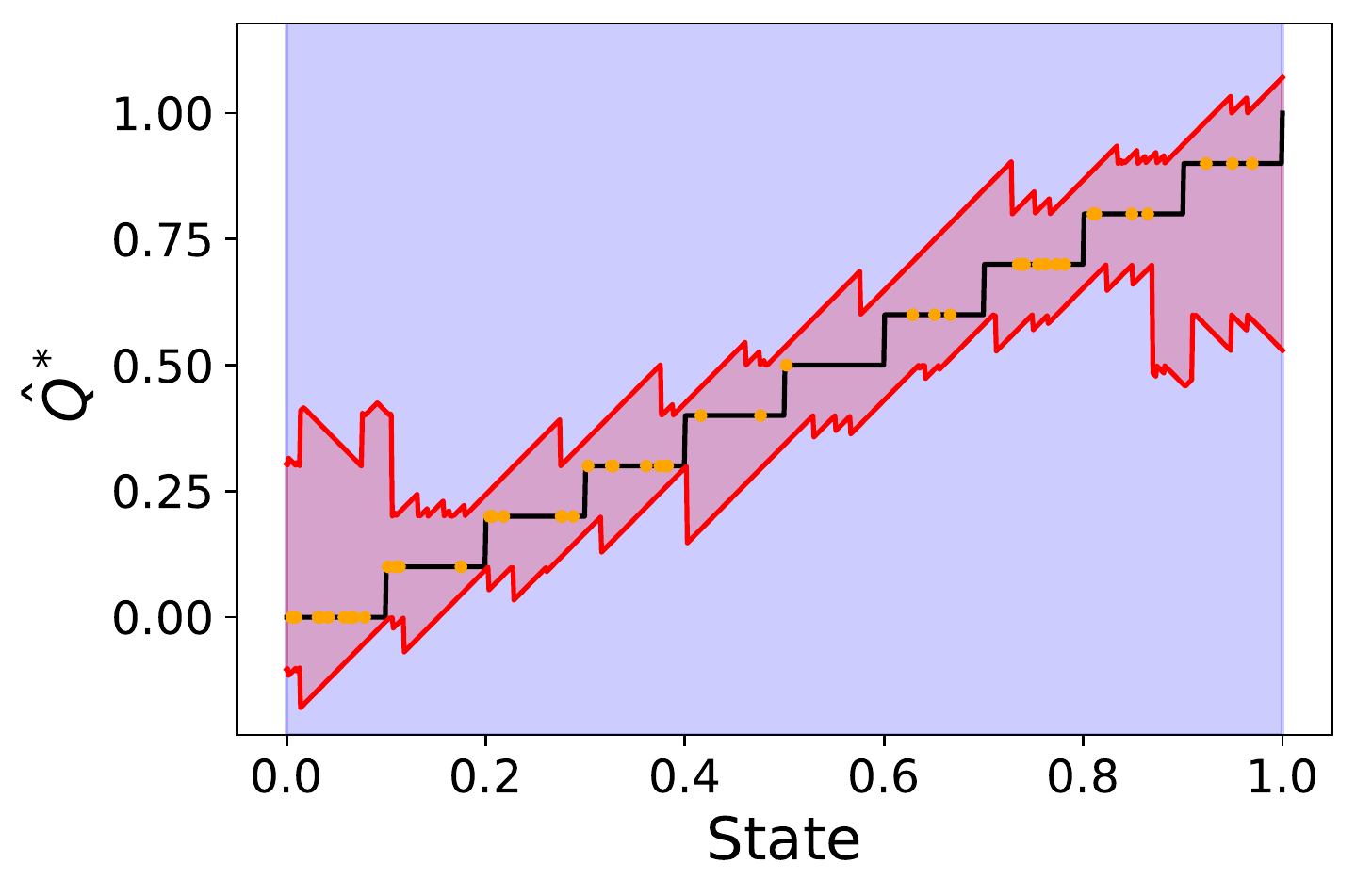}
\caption{Comparison of $L$ (blue) vs.\ $L_{\alpha}$ (red) bounds for a continuous (top) and discontinuous (bottom) function, with different numbers of sampled points.}
\label{fig:lip_estimation}
\end{center}
\vskip -0.2in
\end{figure}

To compare the $L_{\alpha}$ bounds with the traditional Lipschitz bounds, in Fig.~\ref{fig:lip_estimation}, we plot both bounds based on 1, 5, and 50 oracle observations for two functions: a sine function riding on top of a linear slope and a discontinuous step function.\footnote{In Appendix~\ref{appendix:details_of_functions}, we present the exact functions as well as discuss how their $L$ and $L_{\alpha}$ values are computed.} The $L_{\alpha}$ bounds are shown in red (for $\alpha$ equal to the wavelength of the sine function and the width of a step for the stairs function), while the standard Lipschitz bounds are shown in blue for comparison (note that for the stairs function $L$ is infinite, resulting in vacuous bounds, indicated as a blue ``background'').

Fig.~\ref{fig:lip_estimation} demonstrates the strength of using a coarse-grained Lipschitz constant: Because $L_{\alpha}$ could be orders of magnitude smaller than $L$ (even if $L$ is finite), $L_{\alpha}$ provides significantly tighter bounds for points far away from observations. A trade-off, however, is that $L_{\alpha}$ is unable to provide tight bounds close to observed points. (For the sine function, the red zone tends to lie inside of the blue zone, except close to the data points.)

A crucial advantage of the $L_{\alpha}$ formalism is that, unlike the Lipschitz constant, $L_{\alpha}$ can be upper bounded using empirical data. While empirically upper bounding the Lipschitz constant of a function is impossible, as infinitely large gradients may be ``hiding'' at infinitely small length-scales, $L_{\alpha}$ smoothness is insensitive to gradients on scales smaller than $\alpha$, and thus can be upper bounded by sampling the function at a density on the order of $\alpha^{-1}$.

\section{Theoretical Properties of $L_{\alpha}$ Bounds}

We now analyze $L_{\alpha}$ bounds from a pure function representation perspective (independent of RL). We first compare the trade-offs between $L$ and $L_{\alpha}$ bounds, and quantify over what fraction of the space the $L_{\alpha}$ bounds are tighter (Sec.~\ref{sec:l_vs_l_alpha}). In Sec.~\ref{sec:overcoming_cons}, we discuss the weakness of $L_{\alpha}$ bounds and introduce methods to derive even tighter bounds that mitigate these shortcomings.

\subsection{Comparing $L$ and $L_{\alpha}$ Bounds}
\label{sec:l_vs_l_alpha}

We first quantify when using $L_{\alpha}$ to bound a function will provide tighter bounds than the standard Lipschitz bounds. Obviously, if a function $f(x)$ is discontinuous, as is common for Q-functions, $L_{\alpha}$-bounds are trivially better because Lipschitz bounds do not exist. When the Lipschitz bound does exist, Theorem~\ref{thm:compare_est_methods} quantifies over what fraction of $\mathcal{X}$ the $L_{\alpha}$-bound is tighter than the Lipschitz bound. 

Before we can state Theorem~\ref{thm:compare_est_methods} (proof in Appendix~\ref{appendix:compare_est_methods_proof}), we first define several quantities and constants to make the presentation precise. First, the volume in $\mathcal{X}$ of a ball of radius $r$ is $C_D r^D$. For example, for a Euclidean metric in 2 and 3 dimensions, we have $C_D = \pi$ and $C_D = 4\pi/3$, respectively. To generalize the notion of a radius to non-spherical spaces, we also define the linear size of an arbitrarily shaped metric space $\mathcal{X}$, $l_{\mathcal{X}}$, as a function of the volume of $\mathcal{X}$, $V_{\mathcal{X}}$, through the relation $V_{\mathcal{X}} = C_{D_{\mathcal{X}}} (l_{\mathcal{X}})^D$. Thus, if $\mathcal{X}$ is spherical, we have $C_{D_{\mathcal{X}}} = C_D$, but the more general value of $C_{D_{\mathcal{X}}}$ allows for $l_{\mathcal{X}}$ to be defined for spaces of different shapes.

\begin{theorem}
\label{thm:compare_est_methods}
Let $f(x)$ be defined over the metric space $\mathcal{X}$, and assume the value of $f(x)$ is known at $N$ points. Assume $f(x)$ is Lipschitz continuous with Lipschitz constant $L$. For any value of $\alpha$ and its corresponding $L_{\alpha}$ over $f(x)$, define $l_{\alpha}$ to be
\begin{align}
\label{eq:switch_dist_between_est_method_effectiveness}
    l_{\alpha} = \begin{cases}
      \alpha & L \geq L_{\alpha} > L / 3 \\
      \frac{2 L_{\alpha} \alpha}{L - L_{\alpha}} & L / 3 \geq L_{\alpha}. \\
   \end{cases}
\end{align}
Then, if $\frac{C_D}{C_{D_{\mathcal{X}}}} \left( \frac{l_{\alpha}}{l_{\mathcal{X}}} \right)^D N < 1$, the fraction of volume of $\mathcal{X}$ over which the $L_{\alpha}$ bound is tighter than the Lipschitz bound is at least $1 - \frac{C_D}{C_{D_{\mathcal{X}}}} \left( \frac{l_{\alpha}}{l_{\mathcal{X}}} \right)^D N$.
\end{theorem}


Intuitively, Theorem \ref{thm:compare_est_methods} states that, as long as observations are not so numerous as to cover $\mathcal{X}$ such that all observed points are less than a distance of $l_{\alpha}$ away from each other (note that $l_{\alpha} < \alpha$), $L_{\alpha}$-bounds are tighter than Lipschitz across a large fraction of the space. Importantly, Theorem~\ref{thm:compare_est_methods} shows that $L_{\alpha}$-bounds are tighter over a larger fraction of the space in the low data regime (small $N$) and exponentially better with increasing dimensionality. Furthermore, the margin by which $L_{\alpha}$ bounds are tighter than Lipschitz bounds also grows significantly in the low data regime (small $N$) and in high-dimensions. In that sense, the curse of dimensionality increases the relative advantage of $L_{\alpha}$ bounds over Lipschitz bounds.

\subsection{Overcoming the Shortcomings of $L_{\alpha}$ Bounds}
\label{sec:overcoming_cons}

As mentioned earlier, a major drawback of using $L_{\alpha}$ bounds is that, in the vicinity of observed points, they are looser than Lipschitz bounds, if those exist. In fact, even in the limit of infinite data, the uncertainty gap, $\hat{f}^{\alpha}_{UB}(x) - \hat{f}^{\alpha}_{LB}(x)$, will approach $2 L_{\alpha} \alpha$ rather than $0$, for all $x$. This issue can be addressed in two ways. In Sec.~\ref{sec:multiple_alpha}, we provide a method of using multiple values of $\alpha$ and their corresponding $L_{\alpha}$ values. Then, in Sec.~\ref{sec:relax_strictness}, we propose a method that provides tight bounds far from observed points with convergence of the uncertainty gap to zero, at the expense of providing bounds that are not strict, but where we can bound the violations, both in terms of magnitude and over the volume of $\mathcal{X}$ where these violations exist.

\subsubsection{Multiple Smoothness Length-Scales}
\label{sec:multiple_alpha}

\begin{figure}[t]
\vskip 0.2in
\begin{center}
\includegraphics[width=0.48\columnwidth]{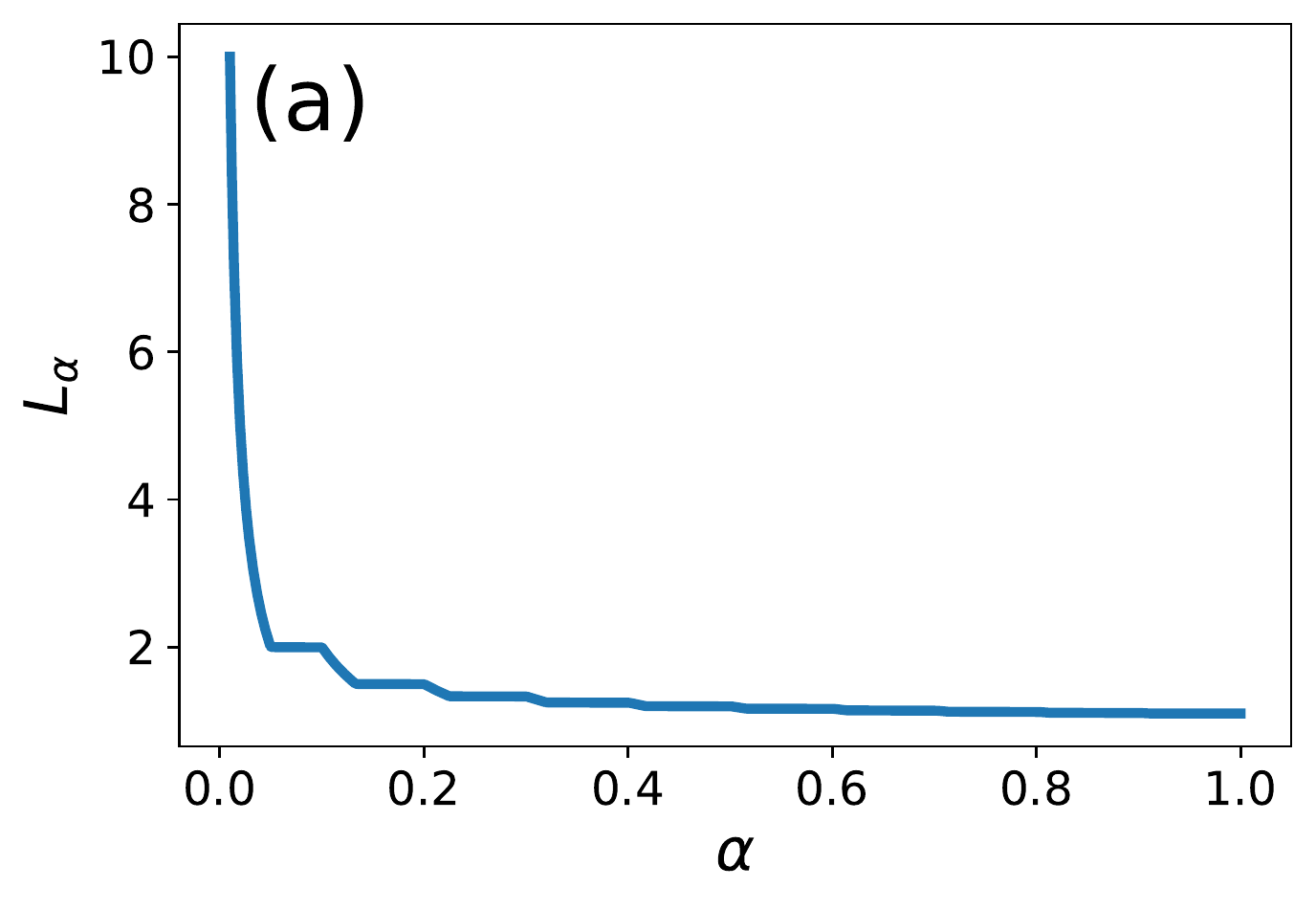}
\includegraphics[width=0.48\columnwidth]{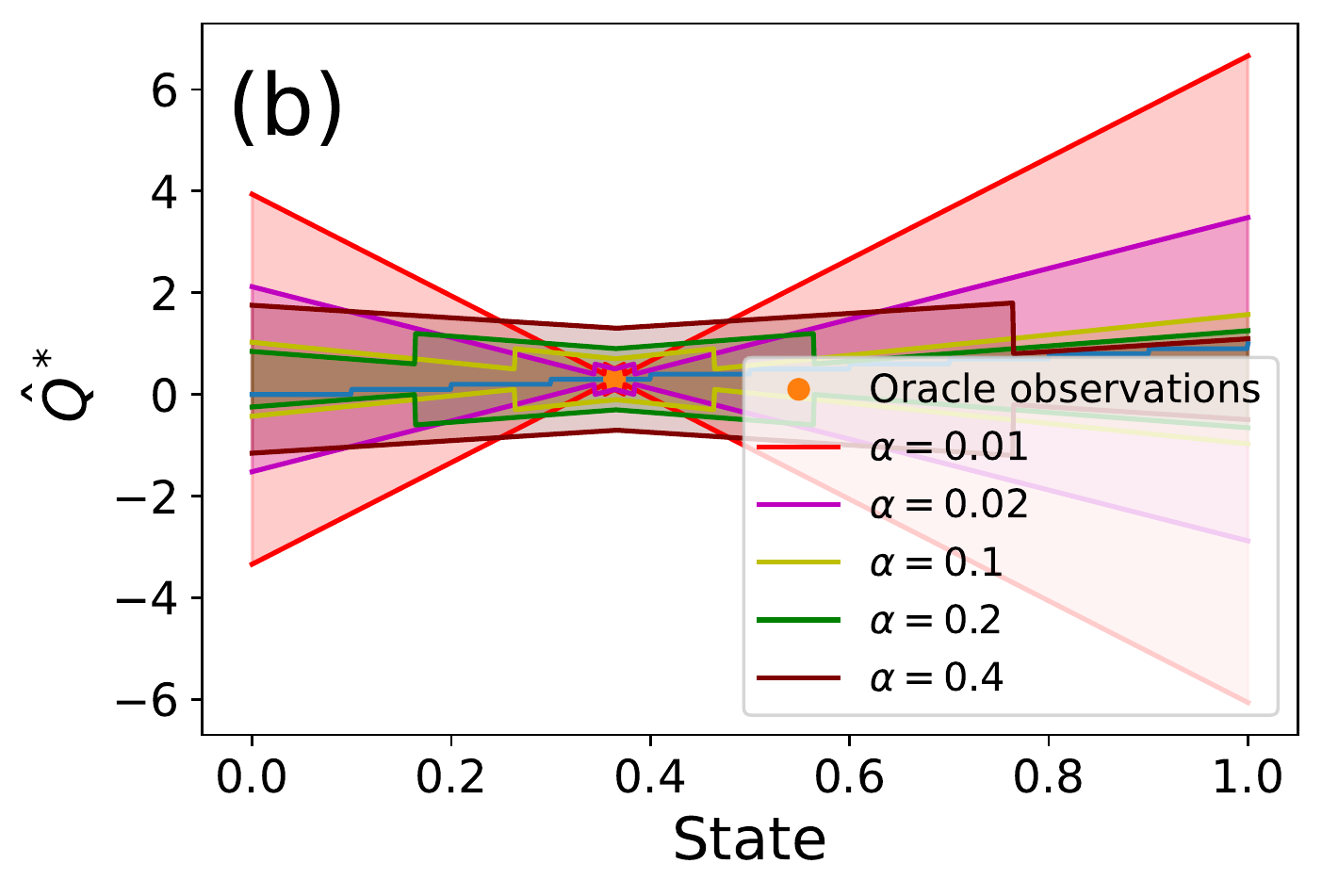}
\includegraphics[width=0.48\columnwidth]{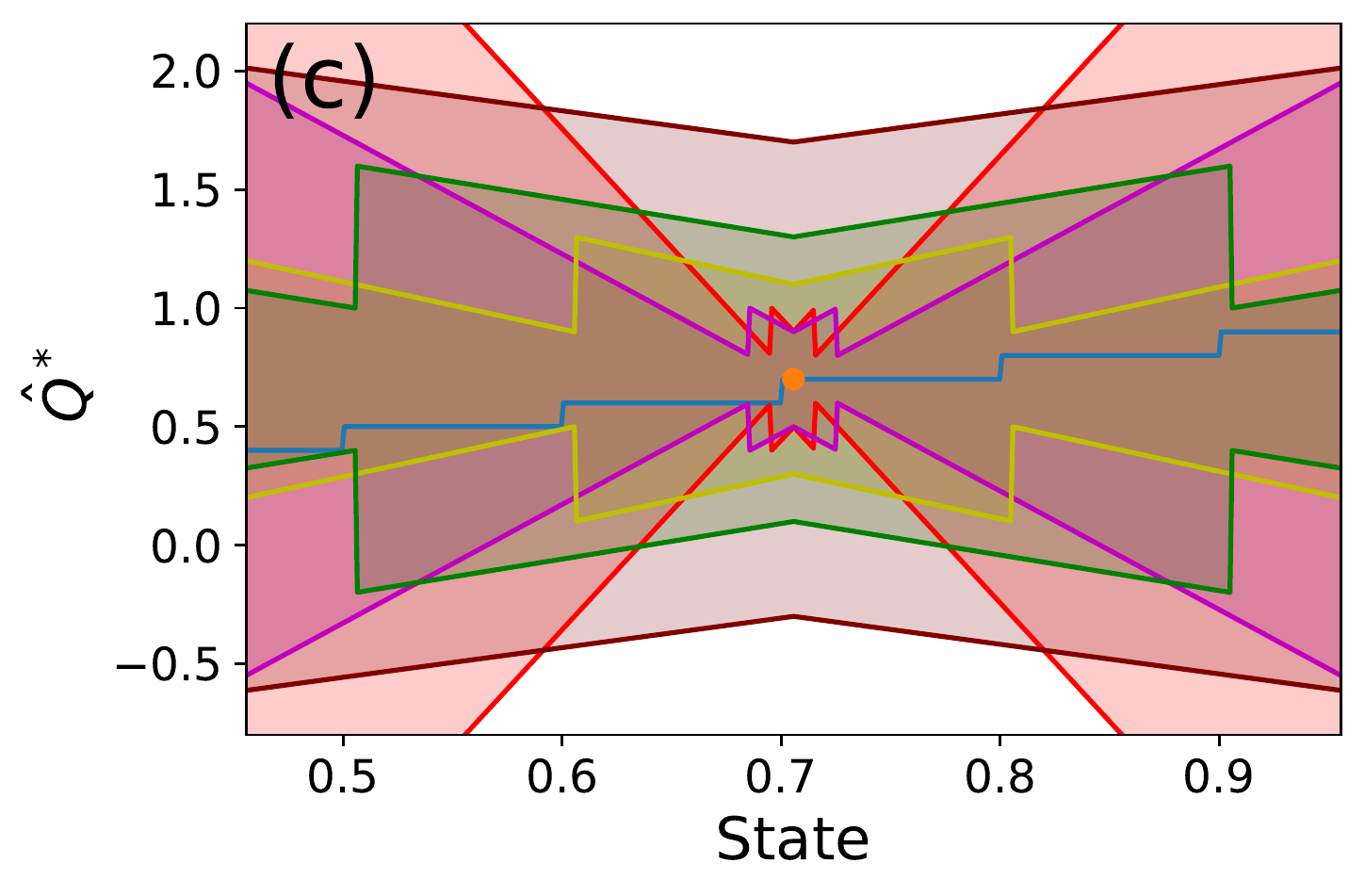}
\includegraphics[width=0.48\columnwidth]{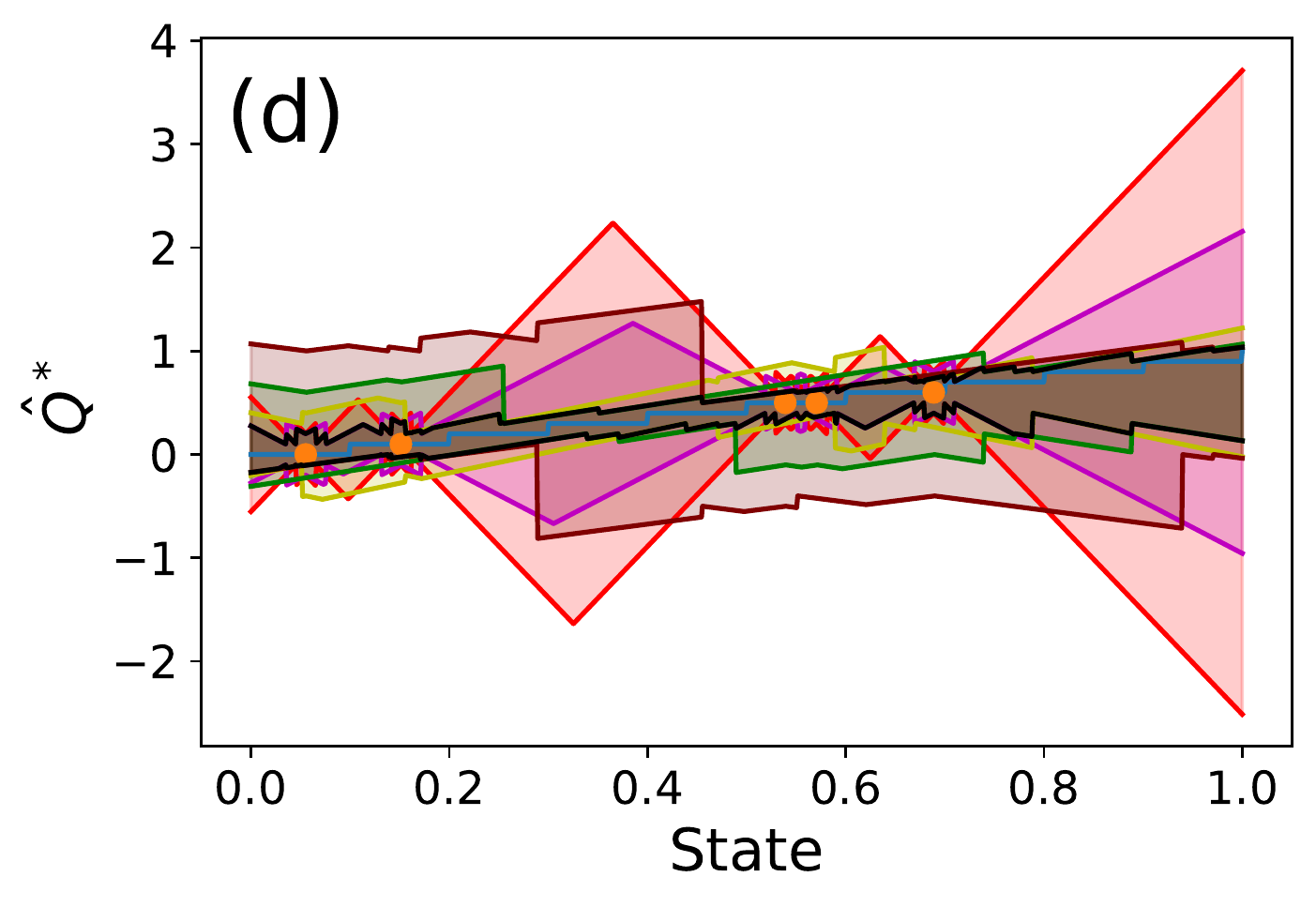}
\caption{(a) $L_{\alpha}$ vs.\ $\alpha$ for the steps function. By combining the $L_{\alpha}$ bounds for different values of $\alpha$ ((b) and zoomed in (c)) we can obtain a much tighter overall bound (d).}
\label{fig:lip_estimation_multiple_alpha}
\end{center}
\vskip -0.2in
\end{figure}

The trade-off between tight bounds close to observed points for $L_{\alpha}$ with small values of $\alpha$, and tight bounds away from observed points for large $\alpha$ values suggests that we can achieve the best of both worlds by combining bounds over different length scales. If we know $L_{\alpha}$ for multiple values of $\alpha$, we can compute the tightest possible upper-bound by taking the minimum of the upper-bounds computed using different $(\alpha, L_{\alpha})$ pairs, and follow a similar procedure for computing a lower bound, giving:
\begin{align}
\label{eq:lip_all_alpha_est}
    \hat{f}_{UB}^{\alpha^*}(x) = \min_{\alpha} \min_{x'} f(x') + g_x^{\alpha}(x') \\ \nonumber
    \hat{f}_{LB}^{\alpha^*}(x) = \max_{\alpha} \max_{x'} f(x') - g_x^{\alpha}(x').
\end{align}
In Fig.~\ref{fig:lip_estimation_multiple_alpha}, we illustrate how such a procedure would work for bounding a step function $f(x)$ for which we have oracular knowledge of $f(x)$ for sampled points. In Fig.~\ref{fig:lip_estimation_multiple_alpha} (a), we plot $L_{\alpha}$ vs.\ $\alpha$,\footnote{In Appendix~\ref{appendix:details_of_functions}, we show how $L_{\alpha}$ can be analytically computed for the step function.} and compare the bounds generated using the value of a single point for five $(\alpha, L_{\alpha})$ pairs in Fig.~\ref{fig:lip_estimation_multiple_alpha} (b-c). In Fig.~\ref{fig:lip_estimation_multiple_alpha} (d), we show in black the resulting bounds when all five  $(\alpha, L_{\alpha})$ pairs are used. Note that, if $f(x)$ is Lipschitz continuous and known, $L$ can also be used for this method.

The major hurdle to using this method is that it requires knowing $L_{\alpha}$ for multiple values of $\alpha$ (in an ideal world for \emph{all} values of $\alpha$), which assumes a lot of domain knowledge. Additionally, even if many $(\alpha, L_{\alpha})$ pairs are known, the computation of the bound could be expensive (especially in the ideal case of knowing $L_{\alpha}$ for all continuous values of $\alpha$). Thus, while Eq.~\ref{eq:lip_all_alpha_est} could provide the tightest bounds possible, the next section provides a practical way to obtain tight bounds in the vicinity of observed points, while retaining the ability to use $L_{\alpha}$ for overall tight bounds, at the expense of allowing bounded violations of the bounds.

\subsubsection{Relaxing the Strictness of the Bounds}
\label{sec:relax_strictness}

An alternative approach that allows us to obtain approximate bounds that become arbitrarily tight in the limit of infinite data is by replacing $L$ in Eq.~\ref{eq:lip_est} with $L_{\alpha}$. This method will retain the advantage of having tight bounds far away from points at which $f(x)$ is known, with the caveat that there might be violations of the bounds in the vicinity of such points. However, by comparison with Eq.~\ref{eq:lip_alpha_est_g}, we see that such violations may be at most $2 \alpha L_{\alpha}$ (i.e., $f(x) - \hat{f}_{UB} < 2 \alpha L_{\alpha}$).

In practice, we have found that this method still leads to good results when used in RL. This property was demonstrated in the experiment presented in Sec.~\ref{sec:intuition}---the results presented are obtained using the algorithm developed in \citet{ni2019learning}, but with different replacement values for $L$. Despite violations of the bounds, the algorithm obtains its best results with an intermediate replacement value for $L$.

\section{Implications for RL}

We now discuss the implications of using $L_{\alpha}$ bounds instead of Lipschitz bounds in RL. We first show in Sec.~\ref{sec:l_alpha_bound} that the structure of MDPs allows us to upper bound $L_{\alpha}$ for one meaningful value of $\alpha$, and thus the ability to use it for RL is guaranteed for any domain over which a suitable distance metric is defined. This fact stands in contrast to the Lipschitz constant, which is not always known and is hard to estimate. In Sec.~\ref{sec:implications_to_existing_works}, we discuss how $L_{\alpha}$ bounds can be incorporated into existing Lipschitz-based algorithms, and what type of theoretical results can be expected to carry over when switching between the bounds.

\subsection{Upper Bounding $L_{\alpha}$}
\label{sec:l_alpha_bound}

A crucial concern when applying Lipschitz (or $L_{\alpha}$) bounds is knowing the specific value of $L$ (or $L_{\alpha}$). For Lipschitz bounds, assuming $L$ is known, as many recent papers do, is a very strong assumption. Estimating the Lipschitz constant is very hard, and, because very steep gradients could potentially exist on very small length scales in $\mathcal{X}$,
it is not possible to derive an upper bound from samples. Such a
bound is required for many optimality results to hold. Fortunately, upper bounds \emph{can} be estimated for the $L_\alpha$ constant: See Prop.~\ref{thm:max_lip_alpha} for an upper bound on $L_{\alpha}$ for $\alpha = d_{\min}$ (proof in Appendix ~\ref{appendix:proof_max_lip_alpha}). This property not only provides an off-the-shelf
value for use in algorithms, it does so for a relevant length scale of the MDP. Consider for simplicity the case of $k=1$, so for all points $x, x' \in \mathcal{X}$ such that $d(x, x') \leq d_{\min}$ the agent can transition between the points in one step. The motivation for defining $L_{\alpha}$ was to have a smoothness measure that is insensitive to fluctuations at small lengthscales of $\mathcal{X}$, and therefore choosing $\alpha = d_{\min}$ smooths over changes in $Q$ on the scale of single transitions, which, as shown in Sec.~\ref{sec:intuition} and~\ref{sec:q_func_structure}, is an important source of discontinuous jumps of the Q-function.

\begin{proposition}
\label{thm:max_lip_alpha}
For an MDP satisfying Assumption~\ref{ass:metric_structure_1}, and convex $\mathcal{X}$, for $\alpha = d_{\min}$, and $\Delta Q^*_{\max}$ defined in Corollary~\ref{thm:max_jump_in_mdp_global}, we can bound $L_{\alpha}$ by $L_{\alpha} \leq \frac{2 \Delta Q^*_{\max}}{d_{\min}}$.
\end{proposition}

While the bound presented in Proposition \ref{thm:max_lip_alpha} may be conservative, it is a marked improvement over the potentially infinite Lipschitz constant value, which cannot be bounded empirically without additional assumptions.



\subsection{Extending $L_{\alpha}$ Bounds to Existing Lipschitz Algorithms and Results}
\label{sec:implications_to_existing_works}

Most Lipschitz-based algorithms use the Lipschitz assumption to generate bounds on the Q-function; by substituting the $L_{\alpha}$ bounds for the Lipschitz bounds, the same algorithms can be used as-is. How would the theoretical guarantees of these algorithms compare with the original algorithms? This discussion can be divided into several cases. First, we have demonstrated that, for a  wide range of common RL tasks, the Lipschitz smoothness assumption on the Q-function is violated, making the argument moot. By contrast, for any bounded domain and a bounded Q-function, $L_{\alpha}$ will always exist for all $\alpha > 0$. Second, even if the Q-function is Lipschitz continuous, $L$ is usually  unknown and hard to estimate or even upper bound. Meanwhile, Prop.~\ref{thm:max_lip_alpha} provides an upper bound for $L_{\alpha}$ for at least one value of $\alpha$ that can readily be used in algorithms. Furthermore, if $L$ is known, any algorithm that relies on using it to upper bound the Q-function can be modified to also incorporate $L_{\alpha}$ for $\alpha=d_{\min}$ using Eq.~\ref{eq:lip_all_alpha_est} at an additional computational factor of at most 2.

Finally, although we can always use both $L_{\alpha}$ and $L$ simultaneously (if $L$ exists and is known), we can compare their performance individually. Early in training, when the agent has observed relatively few points, by Theorem~\ref{thm:compare_est_methods}, $L_{\alpha}$ provides tighter bounds than $L$, and can therefore explore more efficiently by ignoring large regions of the state--action space whose values can be bounded below optimality. However, later in training, when the domain has been thoroughly explored, $L_{\alpha}$ bounds are looser in the vicinity of observed points and are less suitable to fine tuning policies already close to optimality. As a consequence, many of the theoretical results in the literature can be extended to using $L_{\alpha}$ bounds, but care must taken to keep track of when the agent has observed enough of the space that the $L$ bounds become tighter and should be used instead.

As an example, Prop.~\ref{thm:zooming_prop} demonstrates how $L_{\alpha}$ bounds can be applied to the theoretical results of the recently proposed ZoomingRL algorithm~\citep{touati2020zooming} (proof in Appendix~\ref{appendix:proof_zooming}).

\begin{proposition}
\label{thm:zooming_prop}
In the ZoomRL algorithm introduced in \citet{touati2020zooming}, for any value of $\alpha$ such that $L_{\alpha} \leq L/2$, substituting $L_{\alpha}$-bounds for $L$ results in improved regret for at least the first $\frac{2}{3 \alpha}$ episodes.
\end{proposition}

This result underscores a trade-off when using $L_{\alpha}$ for exploration. When $\alpha$ is small, we obtain better regret using $L_{\alpha}$ over more episodes, but the regret gap is small. Conversely, with a large value of $\alpha$, we can obtain larger regret improvement, but that improvement exists for fewer episodes.

Note that quite often $L_\alpha$ is much smaller than $L$ even for relatively small values of $\alpha$ (see, for example, Fig. \ref{fig:lip_estimation_multiple_alpha}(a)), and therefore in practice it would often be easy to find an $\alpha$ such that $L_{\alpha} \leq L/2$.

\section{Conclusion}

Using the smoothness properties of the Q-function is a promising direction for performing RL in continuous state--action spaces. We have shown, however, that care must be taken when making assumptions about smoothness properties, as traditional notions of smoothness do not always hold for such domains. Our coarse-grained smoothness approach takes a first step towards bridging the gap between existing work on Lipschitz-based RL and the realistic Q-function geometry of many MDPs. We have demonstrated that many of the algorithms already presented in the literature can easily be modified to incorporate this new smoothness measure, and that useful theoretical properties of these algorithms can be carried over.

\bibliography{references}

\begin{thebibliography}{27}
\providecommand{\natexlab}[1]{#1}
\providecommand{\url}[1]{\texttt{#1}}
\expandafter\ifx\csname urlstyle\endcsname\relax
  \providecommand{\doi}[1]{doi: #1}\else
  \providecommand{\doi}{doi: \begingroup \urlstyle{rm}\Url}\fi

\bibitem[Arjovsky et~al.(2017)Arjovsky, Chintala, and
  Bottou]{arjovsky2017wasserstein}
M.~Arjovsky, S.~Chintala, and L.~Bottou.
\newblock Wasserstein generative adversarial networks.
\newblock In \emph{International Conference on Machine Learning}, pages
  214--223. PMLR, 2017.

\bibitem[Asadi et~al.(2018)Asadi, Misra, and Littman]{asadi2018lipschitz}
K.~Asadi, D.~Misra, and M.~Littman.
\newblock Lipschitz continuity in model-based reinforcement learning.
\newblock In \emph{International Conference on Machine Learning}, pages
  264--273, 2018.

\bibitem[Berkenkamp et~al.(2017)Berkenkamp, Turchetta, Schoellig, and
  Krause]{safe_model_based}
F.~Berkenkamp, M.~Turchetta, A.~P. Schoellig, and A.~Krause.
\newblock Safe model-based reinforcement learning with stability guarantees.
\newblock In I.~Guyon, U.~von Luxburg, S.~Bengio, H.~M. Wallach, R.~Fergus,
  S.~V.~N. Vishwanathan, and R.~Garnett, editors, \emph{Advances in Neural
  Information Processing Systems 30: Annual Conference on Neural Information
  Processing Systems 2017, December 4-9, 2017, Long Beach, CA, {USA}}, pages
  908--918, 2017.

\bibitem[Chandak et~al.(2020)Chandak, Jordan, Theocharous, White, and
  Thomas]{chandak2020towards}
Y.~Chandak, S.~M. Jordan, G.~Theocharous, M.~White, and P.~S. Thomas.
\newblock Towards safe policy improvement for non-stationary {MDP}s.
\newblock \emph{arXiv preprint arXiv:2010.12645}, 2020.

\bibitem[Finlay et~al.(2019)Finlay, Oberman, and Abbasi]{finlay2019improved}
C.~Finlay, A.~M. Oberman, and B.~Abbasi.
\newblock Improved robustness to adversarial examples using {L}ipschitz
  regularization of the loss, 2019.
\newblock URL \url{https://openreview.net/forum?id=HkxAisC9FQ}.

\bibitem[Gelada et~al.(2019)Gelada, Kumar, Buckman, Nachum, and
  Bellemare]{deep_mdp}
C.~Gelada, S.~Kumar, J.~Buckman, O.~Nachum, and M.~G. Bellemare.
\newblock Deep{MDP}: Learning continuous latent space models for representation
  learning.
\newblock \emph{CoRR}, abs/1906.02736, 2019.

\bibitem[Gottesman et~al.(2019)Gottesman, Liu, Sussex, Brunskill, and
  Doshi{-}Velez]{omer_combining}
O.~Gottesman, Y.~Liu, S.~Sussex, E.~Brunskill, and F.~Doshi{-}Velez.
\newblock Combining parametric and nonparametric models for off-policy
  evaluation.
\newblock In K.~Chaudhuri and R.~Salakhutdinov, editors, \emph{Proceedings of
  the 36th International Conference on Machine Learning, {ICML} 2019, 9-15 June
  2019, Long Beach, California, {USA}}, volume~97 of \emph{Proceedings of
  Machine Learning Research}, pages 2366--2375. {PMLR}, 2019.

\bibitem[Kakade et~al.(2003)Kakade, Kearns, and Langford]{metric_e3}
S.~Kakade, M.~J. Kearns, and J.~Langford.
\newblock Exploration in metric state spaces.
\newblock In \emph{Proceedings of the 20th International Conference on Machine
  Learning (ICML-03)}, pages 306--312, 2003.

\bibitem[Kearns and Singh(2002)]{e3}
M.~Kearns and S.~Singh.
\newblock Near-optimal reinforcement learning in polynomial time.
\newblock \emph{Machine Learning}, 49\penalty0 (2-3):\penalty0 209--232, 2002.

\bibitem[Kleinberg et~al.(2008)Kleinberg, Slivkins, and
  Upfal]{zooming_kleinberg}
R.~Kleinberg, A.~Slivkins, and E.~Upfal.
\newblock Multi-armed bandits in metric spaces.
\newblock \emph{CoRR}, abs/0809.4882, 2008.

\bibitem[Krishnamurthy et~al.(2020)Krishnamurthy, Langford, Slivkins, and
  Zhang]{zooming_akshay}
A.~Krishnamurthy, J.~Langford, A.~Slivkins, and C.~Zhang.
\newblock Contextual bandits with continuous actions: Smoothing, zooming, and
  adapting.
\newblock \emph{J. Mach. Learn. Res.}, 21:\penalty0 137:1--137:45, 2020.
\newblock URL \url{http://jmlr.org/papers/v21/19-650.html}.

\bibitem[Lakshmanan et~al.(2015)Lakshmanan, Ortner, and
  Ryabko]{improved_regret_bounds}
K.~Lakshmanan, R.~Ortner, and D.~Ryabko.
\newblock Improved regret bounds for undiscounted continuous reinforcement
  learning.
\newblock In F.~R. Bach and D.~M. Blei, editors, \emph{Proceedings of the 32nd
  International Conference on Machine Learning, {ICML} 2015, Lille, France,
  6-11 July 2015}, volume~37 of \emph{{JMLR} Workshop and Conference
  Proceedings}, pages 524--532. JMLR.org, 2015.

\bibitem[Lecarpentier and Rachelson(2019)]{lipschitz_non_staionary}
E.~Lecarpentier and E.~Rachelson.
\newblock Non-stationary {M}arkov decision processes, a worst-case approach
  using model-based reinforcement learning.
\newblock In H.~M. Wallach, H.~Larochelle, A.~Beygelzimer,
  F.~d'Alch{\'{e}}{-}Buc, E.~B. Fox, and R.~Garnett, editors, \emph{Advances in
  Neural Information Processing Systems 32: Annual Conference on Neural
  Information Processing Systems 2019, NeurIPS 2019, December 8-14, 2019,
  Vancouver, BC, Canada}, pages 7214--7223, 2019.

\bibitem[Lecarpentier et~al.(2020)Lecarpentier, Abel, Asadi, Jinnai, Rachelson,
  and Littman]{lecarpentier2020lipschitz}
E.~Lecarpentier, D.~Abel, K.~Asadi, Y.~Jinnai, E.~Rachelson, and M.~L. Littman.
\newblock Lipschitz lifelong reinforcement learning.
\newblock \emph{arXiv preprint arXiv:2001.05411}, 2020.

\bibitem[Neyshabur et~al.(2015)Neyshabur, Tomioka, and
  Srebro]{neyshabur2015norm}
B.~Neyshabur, R.~Tomioka, and N.~Srebro.
\newblock Norm-based capacity control in neural networks.
\newblock In \emph{Conference on Learning Theory}, pages 1376--1401. PMLR,
  2015.

\bibitem[Ni et~al.(2019)Ni, Yang, and Wang]{ni2019learning}
C.~Ni, L.~F. Yang, and M.~Wang.
\newblock Learning to control in metric space with optimal regret.
\newblock In \emph{2019 57th Annual Allerton Conference on Communication,
  Control, and Computing (Allerton)}, pages 726--733. IEEE, 2019.

\bibitem[Ortner and Ryabko(2013)]{regret_bounds_for_undiscounted}
R.~Ortner and D.~Ryabko.
\newblock Online regret bounds for undiscounted continuous reinforcement
  learning.
\newblock \emph{CoRR}, abs/1302.2550, 2013.

\bibitem[Osband and Van~Roy(2014)]{osband2014model}
I.~Osband and B.~Van~Roy.
\newblock Model-based reinforcement learning and the eluder dimension.
\newblock \emph{arXiv preprint arXiv:1406.1853}, 2014.

\bibitem[Pazis and Parr(2013)]{pazis_and_parr}
J.~Pazis and R.~Parr.
\newblock {PAC} optimal exploration in continuous space {M}arkov decision
  processes.
\newblock In M.~desJardins and M.~L. Littman, editors, \emph{Proceedings of the
  Twenty-Seventh {AAAI} Conference on Artificial Intelligence, July 14-18,
  2013, Bellevue, Washington, {USA}}. {AAAI} Press, 2013.

\bibitem[Pirotta et~al.(2015)Pirotta, Restelli, and Bascetta]{lipschitz_pg}
M.~Pirotta, M.~Restelli, and L.~Bascetta.
\newblock Policy gradient in {L}ipschitz {M}arkov decision processes.
\newblock \emph{Mach. Learn.}, 100\penalty0 (2-3):\penalty0 255--283, 2015.

\bibitem[Shah and Xie(2018)]{Q_learning_nearest}
D.~Shah and Q.~Xie.
\newblock Q-learning with nearest neighbors.
\newblock In S.~Bengio, H.~M. Wallach, H.~Larochelle, K.~Grauman,
  N.~Cesa{-}Bianchi, and R.~Garnett, editors, \emph{Advances in Neural
  Information Processing Systems 31: Annual Conference on Neural Information
  Processing Systems 2018, NeurIPS 2018, December 3-8, 2018, Montr{\'{e}}al,
  Canada}, pages 3115--3125, 2018.

\bibitem[Sinclair et~al.(2020)Sinclair, Wang, Jain, Banerjee, and
  Yu]{adaptive_discretization_sinclair}
S.~R. Sinclair, T.~Wang, G.~Jain, S.~Banerjee, and C.~L. Yu.
\newblock Adaptive discretization for model-based reinforcement learning.
\newblock In H.~Larochelle, M.~Ranzato, R.~Hadsell, M.~Balcan, and H.~Lin,
  editors, \emph{Advances in Neural Information Processing Systems 33: Annual
  Conference on Neural Information Processing Systems 2020, NeurIPS 2020,
  December 6-12, 2020, virtual}, 2020.

\bibitem[Song and Sun(2019)]{efficient_model_free_in_metric_spaces}
Z.~Song and W.~Sun.
\newblock Efficient model-free reinforcement learning in metric spaces.
\newblock \emph{CoRR}, abs/1905.00475, 2019.

\bibitem[Strehl and Littman(2008)]{strehl2008analysis}
A.~L. Strehl and M.~L. Littman.
\newblock An analysis of model-based interval estimation for {M}arkov decision
  processes.
\newblock \emph{Journal of Computer and System Sciences}, 74\penalty0
  (8):\penalty0 1309--1331, 2008.

\bibitem[Sutton and Barto(2018)]{sutton2018reinforcement}
R.~S. Sutton and A.~G. Barto.
\newblock \emph{Reinforcement Learning: An Introduction}.
\newblock 2018.

\bibitem[Tang et~al.(2020)Tang, Feng, Zhang, Peng, and Liu]{tang2020off}
Z.~Tang, Y.~Feng, N.~Zhang, J.~Peng, and Q.~Liu.
\newblock Off-policy interval estimation with {L}ipschitz value iteration.
\newblock \emph{Advances in Neural Information Processing Systems}, 33, 2020.

\bibitem[Touati et~al.(2020)Touati, Taiga, and Bellemare]{touati2020zooming}
A.~Touati, A.~A. Taiga, and M.~G. Bellemare.
\newblock Zooming for efficient model-free reinforcement learning in metric
  spaces.
\newblock \emph{arXiv preprint arXiv:2003.04069}, 2020.

\end{thebibliography}
\bibliographystyle{icml2020}

\clearpage

\onecolumn
\appendix
\begin{center}
{\LARGE  \bf Coarse-Grained Smoothness for RL in Metric Spaces - Appendix 
}
\end{center}

\appendix

\section{Proof of Theorem \ref{thm:max_jump_in_mdp}}
\label{appendix:proof_max_jump_in_mdp}

\paragraph{Theorem \ref{thm:max_jump_in_mdp}}

Under Assumption~\ref{ass:metric_structure_1},
for any two points $x_1, x_2 \in \mathcal{X}$ such that $d(s_1, s_2) \leq d_{\min}$, the difference $\Delta Q^* \equiv |Q^*(x_1) - Q^*(x_2)|$ is bounded by
\begin{align}
    \Delta Q^* \leq \frac{1 - \gamma^k}{1 - \delta \gamma^k} \left( Q^*_{{\max},(1,2)} - \frac{r_{\min}}{1 - \gamma} \right),
\end{align}
where $Q^*_{{\max},(1,2)} \equiv \max(Q^*(x_1), Q^*(x_2))$.

\begin{proof}
Without loss of generality, assume $Q^*(x_1) > Q^*(x_2)$ and therefore $Q^*_{{\max},(1,2)} = Q^*(x_1)$. To lower bound $\Delta Q^*$, we need to find the lowest possible value of $Q^*(x_2)$. Because $d(s_1, s_2) \leq d_{\min}$, the agent can, with probability $1-\delta$, reach $s_1$ from $s_2$ in at most $k$ steps (Assumption~\ref{ass:metric_structure_1}), and from $s_1$ can follow the optimal policy.

Let $n_k$ be the number of $k$-step attempts needed to reach $s_1$. In the worst case, the agent will collect a (potentially negative) discounted reward of $\frac{1 -  \gamma^{k n_k}}{1-\gamma}r_{\min}$ along the way. Therefore, for a particular $n_k$, the discounted expected reward from $x_2$ is at least $Q_{n_k}^*(x_2) = \frac{1-\gamma^{k n_k}}{1-\gamma}r_{\min} + \gamma^{k n_k} Q^*(x_1)$. Part (ii) of Assumption~\ref{ass:metric_structure_1} is used to make sure that even if the agent does not reach $s_1$ in a particular $k$-step attempt it remains within $d_{\min}$ of it, and can therefore try again until it does (in Appendix~\ref{appendix:assumption_modification} we relax this assumption). The probability of the agent requiring exactly $n_k$ attempts is $\delta^{n_k -1}(1-\delta)$, and therefore the value of $Q^*(x_2)$ is at least the expectation over $n_k$ of $Q_{n_k}^*(x_2)$:

\begin{align}
\label{eq:q2_min}
Q^*(x_2) &\geq \sum_{n_k=1}^{\infty} \delta^{n_k -1}(1-\delta) Q_{n_k}^*(x_2) \\ \nonumber
&= \sum_{n_k=1}^{\infty} \delta^{n_k -1}(1-\delta) \left( \frac{1-\gamma^{k n_k}}{1-\gamma}r_{\min} + \gamma^{k n_k} Q^*(x_1) \right) \\ \nonumber
&= \frac{(1-\delta) r_{\min}}{1-\gamma}\sum_{n_k=1}^{\infty} \delta^{n_k -1} 
-\frac{(1-\delta) r_{\min}}{1-\gamma} \sum_{n_k=1}^{\infty} \delta^{n_k -1} \gamma^{k n_k} 
+ (1 - \delta) Q^*(x_1) \sum_{n_k=1}^{\infty} \delta^{n_k -1} \gamma^{k n_k} \\ \nonumber
&= \frac{(1-\delta) r_{\min}}{(1-\gamma)(1-\delta)} - \frac{\gamma^k (1-\delta) r_{\min}}{(1-\gamma)(1-\delta \gamma^k)} 
+ \frac{\gamma^k (1 - \delta) Q^*(x_1)}{(1-\delta \gamma^k)} \\ \nonumber
&= \frac{(1-\delta) r_{\min}}{(1-\gamma)} \left( \frac{1}{1-\delta} - \frac{\gamma^k}{1 - \delta \gamma^k} \right) 
+ \frac{\gamma^k (1 - \delta) Q^*(x_1)}{(1-\delta \gamma^k)} \\ \nonumber
&= \frac{(1-\delta) r_{\min}}{(1-\gamma)} \frac{(1- \delta \gamma^k - \gamma^k + \delta \gamma^k)}{(1-\delta)(1 - \delta \gamma^k)} 
+ \frac{\gamma^k (1 - \delta) Q^*(x_1)}{(1-\delta \gamma^k)} \\ \nonumber
&= \frac{r_{\min}(1 - \gamma^k)}{(1-\gamma)(1-\delta \gamma^k)} + \frac{\gamma^k (1 - \delta) Q^*(x_1)}{(1-\delta \gamma^k)}
\end{align}

Now to obtain the largest difference between the values of the states we simply subtract the result of Eq.~\ref{eq:q2_min} from $Q^*(x_1)$ to obtain:

\begin{align}
Q^*(x_1) - Q^*(x_2) &\leq Q^*(x_1) - \frac{r_{\min}(1 - \gamma^k)}{(1-\gamma)(1-\delta \gamma^k)} 
- \frac{\gamma^k (1 - \delta) Q^*(x_1)}{(1-\delta \gamma^k)} \\ \nonumber
&= Q^*(x_1) \frac{(1 - \delta \gamma^k - \gamma^k + \delta \gamma^k)}{1 - \delta \gamma^k} 
- \frac{r_{\min}(1 - \gamma^k)}{(1-\gamma)(1-\delta \gamma^k)} \\ \nonumber
&= \frac{1 - \gamma^k}{1 - \delta \gamma^k} \left( Q^*(x_1) - \frac{r_{\min}}{1-\gamma} \right) \\ \nonumber
&= \frac{1 - \gamma^k}{1 - \delta \gamma^k} \left( Q^*_{\max, (1,2)} - \frac{r_{\min}}{1-\gamma} \right),
\end{align}

where the last substitution is due to $Q^*(x_1) > Q^*(x_2)$.

\end{proof}

\section{Modifications of Assumption \ref{ass:metric_structure_1}}
\label{appendix:assumption_modification}

Below we present two variations to Assumption~\ref{ass:metric_structure_1} which may relax it further and allows application to a wider range of problems. We first discuss a variation of the Assumption which allows for applications to continuous stochastic domains (Appendix~\ref{appendix:assumption_modification_stochastic_continuous}). We then present a variation which may alleviate the restrictiveness of part (ii) of Assumption~\ref{ass:metric_structure_1}, that with probability 1 the agent stays within $d_{\min}$ of the state it is trying to reach (Appendix~\ref{appendix:assumption_modification_possibility_of_getting_far_from_goal}).

\subsection{Extension to stochastic continuous domains}
\label{appendix:assumption_modification_stochastic_continuous}

Assumption~\ref{ass:metric_structure_1} is sufficient to derive a bound on the difference of the $Q^*$ function between two states within a given distance from each other for deterministic continuous domains or stochastic discrete domains when a suitable metric is provided (or stochastic continuous domains under discretization). However, Assumption\ref{ass:metric_structure_1} is not suitable for fully stochastic continuous domains. To see that consider an environment in which the action of the agent determines the next state of the agent up to some Gaussian noise, i.e. $s_{t+1} = s_t + a_t + \varepsilon_t$, with $\varepsilon_t \sim \mathcal{N}(0, \sigma^2)$. In this case the probability of reaching \emph{any} specific state is zero, as the transition function is now a probability density rather than a probability mass function.

To alleviate this issue we provide the following alternative to Assumption~\ref{ass:metric_structure_1}:

\begin{assumption}
\label{ass:metric_structure_1_stochastic_continuous}
For any two states $s, s' \in \mathcal{S}$ such that $d(s, s') \leq d_{\min}$, we have that (i)  with probability $1-\delta$, a ball with center $s'$ and radius $d_{\varepsilon}$ can be reached from $s$ in at most $k$ time-steps, and (ii) with probability 1, the agent can remain within a distance of less than $d_{\min}$ from $s'$.
\end{assumption}

While the assumption above can be applied to stochastic continuous domains, it is not enough to prove an analogous result to Theorem~\ref{thm:max_jump_in_mdp}, because there is no way to relate the value at a state $s'$, and the values at states in a ball of radius $d_{\varepsilon}$ around it. For that we must make an additional assumption:

\begin{assumption}
\label{ass:metric_structure_local_q_gap}
There exists a real number $\Delta Q^*_{d_{\varepsilon}} > 0$, such that for any two point $x_1, s_2 \in \mathcal{X}$ satisfying $d(s_1, s_2) \leq d_{\varepsilon}$, we have that $\Delta Q^*_{d_{\varepsilon}} \leq |Q^*(x_1) - Q^*(x_2)|$.
\end{assumption}

We will shortly discuss the limitations of Assumption~\ref{ass:metric_structure_local_q_gap}, but we first note that Assumptions~\ref{ass:metric_structure_local_q_gap} and \ref{ass:metric_structure_1_stochastic_continuous} lead to the following alternative to Theorem~\ref{thm:max_jump_in_mdp}:

\begin{theorem}
\label{thm:max_jump_in_mdp_stochastic_continuous}
Under Assumptions~\ref{ass:metric_structure_1_stochastic_continuous} and~\ref{ass:metric_structure_local_q_gap},
for any two points $x_1, x_2 \in \mathcal{X}$ such that $d(s_1, s_2) \leq d_{\min}$, the difference $\Delta Q^* \equiv |Q^*(x_1) - Q^*(x_2)|$ is bounded by
\begin{align}
    \Delta Q^* \leq \frac{1 - \gamma^k}{1 - \delta \gamma^k} \left( Q^*_{{\max},(1,2)} - \frac{r_{\min}}{1 - \gamma} \right) + \frac{\gamma^k (1 - \delta)}{1 - \delta \gamma^k} \Delta Q^*_{d_{\varepsilon}},
\end{align}
where $Q^*_{{\max},(1,2)} \equiv \max(Q^*(x_1), Q^*(x_2))$.
\end{theorem}

\begin{proof}
The proof is similar to the proof for Theorem~\ref{thm:max_jump_in_mdp}, with the modification that once the agent reaches a ball with center $s_1$ and radius $d_{\varepsilon}$, its value is at least $Q^*(x_1) - \Delta Q^*_{d_{\varepsilon}}$ (rather than equaling $Q^*(x_1)$ when exactly reaching $s_1$).
\end{proof}

Assumption~\ref{ass:metric_structure_local_q_gap} may seem very limiting as it appears to make a local smoothness assumption on the $Q$-function which is similar to the types of assumptions we criticize in this paper. While this is partly true and unfortunately necessary to obtain results for stochastic continuous domains, note that Assumption~\ref{ass:metric_structure_local_q_gap} is still significantly weaker than a Lipschitz-smoothness assumption, as it still allows for discontinuous jumps of the $Q$-function, and makes no assumptions on points farther than $d_{\varepsilon}$ away. Generally $d_{\varepsilon}$ is assumed to be small (on the order of the transition noise).

A different way to phrase such critique of Assumption~\ref{ass:metric_structure_local_q_gap} is to note that it makes an assumption to bound $|Q^*(x_1) - Q^*(x_2)|$ for points within a given distance from each other $(d(s_1, s_2) < d_{\varepsilon})$, to obtain a bound for $|Q^*(x_1) - Q^*(x_2)|$ for points at a different distance from each other $(d(s_1, s_2) < d_{\min})$. This, in fact, leads to demonstrating the value of using Assumption~\ref{ass:metric_structure_local_q_gap} to obtain Theorem~\ref{thm:max_jump_in_mdp_stochastic_continuous}. Note that the same method of proof used for Lemma~\ref{thm:max_q_diff_any_points} in Appendix~\ref{appendix:proof_max_lip_alpha} can be used to get a more straight forward bound on $|Q^*(x_1) - Q^*(x_2)|$ when $d(s_1, s_2) < d_{\min}$, namely:

\begin{align}
\label{eq:naive_bound_on_q_gap}
    |Q^*(x_1) - Q^*(x_2)| \leq \Delta Q^*_{d_{\varepsilon}} \ceil*{\frac{d_{\min}}{d_{\varepsilon}}}.
\end{align}

In many practical settings, the length-scale over which actions affect the transition dynamics is much larger than the length-scale of the noise, i.e. $d_{\min} \gg d_{\varepsilon}$. Under such conditions, $\ceil*{\frac{d_{\min}}{d_{\varepsilon}}}$ will be very large and the bound in Theorem~\ref{thm:max_jump_in_mdp_stochastic_continuous} will be much tighter than the bound given by Eq.~\ref{eq:naive_bound_on_q_gap}. Thus, Theorem~\ref{thm:max_jump_in_mdp_stochastic_continuous} makes use of a local smoothness assumption on the length-scale of the noise to obtain a bound on the length-scale of the dynamics due to actions, which is significantly tighter than the bound which can be obtained without making use of Assumption~\ref{ass:metric_structure_1_stochastic_continuous}.

\subsection{Relaxing part (ii) of Assumption~\ref{ass:metric_structure_1}}
\label{appendix:assumption_modification_possibility_of_getting_far_from_goal}

We now provide an alternative to Assumption~\ref{ass:metric_structure_1} which replaces the requirement that with probability 1 the agent stays within $d_{\min}$ of the state it is trying to reach. The motivation for the original assumption is that when the agent is trying to reach a state $s'$, which it can reach within $k$ steps with high probability, if it fails, it should at the very least not get further away from the state. Here we discuss an alternative assumption, which allows for the possibility of the agent unfortunately ending up further away from the state it is trying to reach. We will assume that if the agent is unable to reach state $s'$ from state $s$ within $k$ steps, the state it will end up in, $s''$, satisfies $d(s, s'') \leq d_{\max}$. We constrain the value of $d_{\max}$ in Assumption~\ref{ass:metric_structure_1_possibility_of_getting_far_from_goal}, and in the proof of Theorem~\ref{thm:max_jump_in_mdp_possibility_of_getting_far_from_goal} justify this constraint by showing that if it is violated, the expected time for an agent to transition between two states that are less than $d_{\min}$ away from each other could be infinite. Thus, the constraint on $d_{\max}$ ensures our requirement of the metric to quantify how quickly an agent can transition between states.

\begin{assumption}
\label{ass:metric_structure_1_possibility_of_getting_far_from_goal}
For any two states $s, s' \in \mathcal{S}$ such that $d(s, s') \leq d_{\min}$, we have that (i)  with probability $1-\delta$, $s'$ can be reached from $s$ in at most $k$ time-steps, and (ii) if the agent does not reach $s'$ within $k$ time-steps, its state after $k$ steps, $s''$, satisfies $d(s, s'') \leq d_{\max} \equiv \frac{1 - \delta}{\delta} d_{\min}$.
\end{assumption}

Note that $\delta$ is assumed to be small, and therefore $d_{\max} > d_{min}$. Under Assumption~\ref{ass:metric_structure_1_possibility_of_getting_far_from_goal}, we can prove the following equivalent result to Theorem~\ref{thm:max_jump_in_mdp}:

\begin{theorem}
\label{thm:max_jump_in_mdp_possibility_of_getting_far_from_goal}
Under Assumption~\ref{ass:metric_structure_1_possibility_of_getting_far_from_goal}, if $r_{\min} \leq 0$,
for any two points $x_1, x_2 \in \mathcal{X}$ such that $d(s_1, s_2) \leq d_{\min}$, the difference $\Delta Q^* \equiv |Q^*(x_1) - Q^*(x_2)|$ is bounded by
\begin{align}
    \Delta Q^* \leq (1 - \gamma^{\tilde{k}}) \left( Q^*_{{\max},(1,2)} - \frac{r_{\min}}{1 - \gamma} \right)
\end{align}
where $Q^*_{{\max},(1,2)} \equiv \max(Q^*(x_1), Q^*(x_2))$ and $\tilde{k} \equiv \frac{k}{(1-\delta) - \ceil*{\frac{d_{\max}}{d_{\min}}} \delta} $.
\end{theorem}

\begin{proof}
Proving Theorem~\ref{thm:max_jump_in_mdp_possibility_of_getting_far_from_goal} requires a different approach than Theorem~\ref{thm:max_jump_in_mdp}. If we try to use the same logic, with probability $(1 - \delta)$ the agent reaches state $s'$ within $k$ steps. However, if that is not the case, the agent may end up as far away as $d_{\max} + d(s, s')$ away from $s''$. To use the same method of proof, we can break $d(s, s'')$ into $\ceil*{\frac{d_{\max}}{d_{\min}}}$ intervals of length $d_{\min}$, and assume the agent can transition between the edges of these intervals every $k$ steps with probability $(1 - \delta)$. Thus, the problem can be thought of as a 1D random walker which starts at a distance of one interval away from the goal, and with probability $(1 - \delta)$ moves one step toward the goal (for this random walker, all ``steps'' are in units of $d_{\min}$ and each time-step is equivalent to $k$ time-steps of the agent), and with probability $\delta$ moves (in the worst case) $D_{d_{\max}} \equiv \ceil*{\frac{d_{\max}}{d_{\min}}}$ steps away from the goal. Denote the location of the random walker $D$, with the goal at $D=0$ and the starting point at $D=1$ (equivalent to at most $d_{\min}$ away from the goal). The expected number of time steps to reach the goal from position $D$, denoted $\tau(D)$, obeys the following equations:

\begin{gather}
\label{eq:random_walker}
\tau(D) = (1 - \delta) (1 + \tau(D - 1)) + \delta (1 + \tau(D + D_{d_{\max}})) \\
\label{eq:random_walker_boundary}
\tau(0) = 0
\end{gather}

These equations can be solved by considering the mean field approximation of the random walker as moving in continuous space with velocity $u = (1-\delta)  (-1) + \delta D_{d_{\max}}$. For the agent to have a net velocity in the direction of the goal, we must require $u < 0$, which gives us the constraint on $d_{\max}$ stated in Assumption~\ref{ass:metric_structure_1_possibility_of_getting_far_from_goal}:

\begin{align}
    (1-\delta) (-1) + \delta D_{d_{\max}} < 0 &\Rightarrow D_{d_{\max}}  < \frac{1 - \delta}{\delta} \\ \nonumber
    &\Rightarrow \ceil*{\frac{d_{\max}}{d_{\min}}} < \frac{1 - \delta}{\delta}
\end{align}

which is satisfied when $d_{\max} < \frac{1 - \delta}{\delta} d_{\min}$. Intuitively this constraint ensures that even if there is a possibility of stochasticity transitioning the agent away rather than towards a state it is trying to reach, the probability of this happening and the amount by which the agent is distanced from its desired state are balanced in such a way that the agent can eventually reach its desired state. If that is not the case, the metric does not satisfy the notion we desire---quantifying how quickly states can be reached from other nearby states.

For the continuous mean-field dynamics, $\tau(D)$ is equivalent to the time required to traverse a distance of $(-D)$ and is given by

\begin{align}
\label{eq:random_walker_solution}
\tau(D) = \frac{D}{1 - \delta - \delta D_{d_{\max}}}.
\end{align}

Note that by substitution we can verify that Eq.~\ref{eq:random_walker_solution} is indeed a solution to Eq.~\ref{eq:random_walker}, and therefore the mean field solution is valid for the discrete dynamics as well. The expected number of time-steps in which an agent will reach $s'$ from a state $s$ such that $d(s, s') \leq d_{\min}$ can now be expressed by

\begin{align}
\label{eq:expected_steps_to_s_prime}
\tilde{k} \equiv k \tau(1) = \frac{k}{1 - \delta - \delta D_{d_{\max}}} = \frac{k}{(1-\delta) - \ceil*{\frac{d_{\max}}{d_{\min}}} \delta}
\end{align}

The logic of the proof for Theorem~\ref{thm:max_jump_in_mdp} can be applied here---without loss of generality, assume $Q^*(x_1) > Q^*(x_2)$, and therefore the agent can transition from $s_2$ to $s_1$ in $\tilde{k}$ steps, collecting a minimum reward of $r_{\min} \frac{1 - \gamma^{\tilde{k}}}{1-\gamma}$, and from state $s_1$ follow the policy for the discounted Q-value of state action pairs in state $s_1$ of $\gamma^{\tilde{k}} Q^*(x_1)$. We thus obtain that for the expected number of steps to reach the goal, $Q_{\tilde{k}}^*(x_2) \geq r_{\min} \frac{1 - \gamma^{\tilde{k}}}{1-\gamma} + \gamma^{\tilde{k}} Q^*(x_1)$. Unfortunately, this is not a true lower bound for $Q^*(x_2)$ because we used an expression lower bounding $Q^*(x_2)$ for a given number of steps to the goal and substituted the expectation of the number of steps, rather than take an expectation over the number of steps for the lower bound of $Q^*(x_2)$. However, for $r_{\min} \leq 0$, the expression $r_{\min} \frac{1 - \gamma^{k}}{1-\gamma} + \gamma^{k} Q^*(x_1)$ is a convex function of $k$, and therefore by the Jensen inequality

\begin{align}
    Q^*(x_2) \geq Q_{\tilde{k}}^*(x_2) \geq r_{\min} \frac{1 - \gamma^{\tilde{k}}}{1-\gamma} + \gamma^{\tilde{k}} Q^*(x_1).
\end{align}

Subtracting from $Q^*(x_1)$ we get

\begin{align}
    |Q^*(x_1) - Q^*(x_2)| &\geq Q^*(x_1) - \left( r_{\min} \frac{1 - \gamma^{\tilde{k}}}{1-\gamma} + \gamma^{\tilde{k}} Q^*(x_1 \right) \\ \nonumber
    &= (1 - \gamma^{\tilde{k}}) \left( Q^*_{{\max},(1,2)} - \frac{r_{\min}}{1 - \gamma} \right)
\end{align}

\end{proof}

\section{Functional forms and smoothness measures of illustrative functions}
\label{appendix:details_of_functions}

\paragraph{Sine-like function.} The analytic form of the function illustrated in Figure \ref{fig:lip_estimation} in the main text (left column) is

\begin{align}
    f(x) = A \sin(2 \pi \omega x) + m x.
\end{align}

Specifically, in Figure \ref{fig:lip_estimation}, the parameter values used are $A=3$, $\omega=2$ and $m=5$. The Lipschitz constant of the function is the largest slope, or the maximum of the first derivative of $f(x)$,

\begin{align}
    f'(x) = 2 \pi \omega A \cos(2 \pi \omega x) + m,
\end{align}

which is attained for all integer values of $x$:

\begin{align}
    L = 2 \pi \omega A + m.
\end{align}

To compute $L_{\alpha}(\alpha)$ we write the slope between any two points $x_1$ and $x_2$ and look for its maximum when the distance between them is larger than $\alpha$. In other words, we wish to maximize $M(x_1, x_2) = |f(x_1) - f(x_2)| / |x_1 - x_2|$ subject to $|x_1 - x_2| \geq \alpha$:

\begin{align}
\label{eq:slope_for_sine}
    M&(x_1, x_2) = \\ \nonumber
    &=\left |\frac{A [\sin(2 \pi \omega x_1) - \sin(2 \pi \omega x_2)] + m (x_1 - x_2)}{x_1 - x_2} \right|  \\ \nonumber
    &=\left |\frac{A [\sin(2 \pi \omega x_1) - \sin(2 \pi \omega x_2)]}{x_1 - x_2} + m \right|  \\ \nonumber
    &=\left |\frac{2A [\sin(\pi \omega \{x_1 - x_2\}) \cos(\pi \omega \{x_1 + x_2\})]}{x_1 - x_2} + m \right|  \\ \nonumber
    &=\left |\frac{2A \pi \omega [\sin(y_1) \cos(y_2)]}{y_1} + m \right|,
\end{align}

where in the last equality we substituted $y_1=\pi \omega \{x_1 - x_2\}$ and $y_2 = \pi \omega \{x_1 + x_2\}$. For simplicity we will assume $m>0$, which implies the largest $M$ obtained will be positive. This assumption allows us to remove the absolute value for easier differentiation.

To find the maximum of $M$ we take it's partial derivatives:

\begin{align}
    \frac{\partial M}{\partial y_1} &= 2 \pi \omega \cos(y_2) \left[ \frac{\cos(y_1)}{y_1} - \frac{\sin(y_1)}{y_1^2} \right] \\
    \frac{\partial M}{\partial y_2} &= - 2 \pi \omega \frac{\sin(y_1) \sin(y_1)}{y_1}
\end{align}

Equating both derivatives to zero provides two groups of solutions:

\emph{Group I}:$$
    \sin(y_1) = 0; \ \ 
    \cos(y_2) = 0$$

Substituting the solutions for group I in Equation \ref{eq:slope_for_sine} yields $M = m$. Once we examine the solutions for Group II, we will see that for every cycle of the sine, we have a solution from Group II with $M > m$.

\emph{Group II}:$$
    y_1 = \tan{y_1}; \ \ 
    \sin(y_2) = 0$$
    
Substituting the solutions for group II in Equation \ref{eq:slope_for_sine} yields

\begin{align}
    M = 2 A \pi \omega \cos(y_2) \cos(y_1) + m.
\end{align}

Because $\sin(y_2) = 0$ implies $\cos(y_2) = \pm 1$, and therefore because we are interested in the maxima of $M$ we choose the sign such that the first term is positive (this can always be done because $y_1$ and $y_2$ can be chosen independently). The maxima of $M$ can then be found by substituting the solutions (which can be found numerically) for $y_1 = \tan{y_1}$ in

\begin{align}
    M = 2 A \pi \omega |\cos(y_1)| + m.
\end{align}

For larger $y_1$, the closest solution of $y_1 = \tan{y_1}$ is closer to $\pi/2 + n \pi$ for an integer $n$, and therefore the corresponding $\cos(y_1)$ is closer to zero. This means that

\begin{align}
    L_{\alpha} = 2 A \pi \omega |\cos(y_*)| + m,
\end{align}

with $y_*$ being the first solution of $y_1 = \tan{y_1}$ such that $y_1 \geq \pi \omega \alpha$ (which is equivalent to the constraint $x_1 - x_2 \geq \alpha$. In the example shown in Figure \ref{fig:lip_estimation} in the main text, where we take $\alpha = \omega^{-1}$, we take the second solution of $y_1 = \tan{y_1}$ which is $y_* \approx 4.49$ and yields:

\begin{align}
    L_{\alpha=\omega} \approx 2 A \pi \omega |\cos(4.49)| + m \approx 13.29.
\end{align}

\paragraph{Stairs function}

The stairs function used for demonstration is a stairs function with a step of magnitude $0.1$ every interval of $0.1$ in the $x$ direction:

\begin{align}
    f(x) = A \floor*{\frac{x}{w}}
\end{align}

with $A = 0.1$ and $w = 0.1$. The function is obviously not continuous and therefore $L$ does not exist (or is considered ``infinite'').

To find $L_{\alpha}$, we must find two points that maximize $M(x_1, x_2) = |f(x_1) - f(x_2)| / |x_1 - x_2|$ subject to $|x_1 - x_2| \geq \alpha$. To understand the possible candidates for such pairs of points, let us first consider that $[x_2, x_1]$ forms an interval that contains $n$ steps  (without loss of generality, assume $x_1 > x_2$, and denote $\Delta \equiv x_1 - x_2$). In such a case we have $f(x_1) - f(x_2) = n A$, and $M = n A / \Delta$. Given an interval which contains $n$ steps, we would like to make $\Delta$ as small as possible in order to maximize $M$. This is achieved for $\Delta = (n - 1) w + \epsilon$, with $\epsilon > 0$ as small as possible (to see how this is true, consider that in a width of slightly more than the width of an entire stair we can include two jumps). Then we have

\begin{align}
\label{eq:l_alpha_for_n_steps}
M = \lim_{\epsilon \rightarrow 0} \frac{n A}{(n - 1) w + \epsilon} = \frac{n A}{(n - 1) w}.
\end{align}

It is clear from Equation \ref{eq:l_alpha_for_n_steps} that $M$ decreases with $n$, and so for a given $\alpha$, one potential choice for $x_1$ and $x_2$ would be choosing them in such a way that they form an interval which contains $\ceil{\alpha / w} + 1$ steps, for which we have $M = \frac{A (\ceil{\alpha/w} + 1)}{\ceil{\alpha/w} w} = \frac{A}{w} \left( 1 + \frac{1}{\ceil{\alpha/w}} \right)$.

The only other potential choice for the interval $[x_2, x_1]$ is to choose its width to be exactly $\alpha$. In that case we can fit $\ceil{\alpha/w}$ steps in it, and obtain $M = \frac{A \ceil{\alpha / w}}{\alpha}$.

Thus, for the stairs function we have:

\begin{align}
    L_{\alpha}(\alpha) = \max \left[ \frac{A}{w} \left( 1 + \frac{1}{\ceil{\alpha/w}} \right), \frac{A}{\alpha} \ceil{\alpha / w} \right].
\end{align}

\section{Proof of Theorem \ref{thm:compare_est_methods}}
\label{appendix:compare_est_methods_proof}

\begin{figure}[t]
\vskip 0.2in
\begin{center}
\includegraphics[width=0.32\columnwidth]{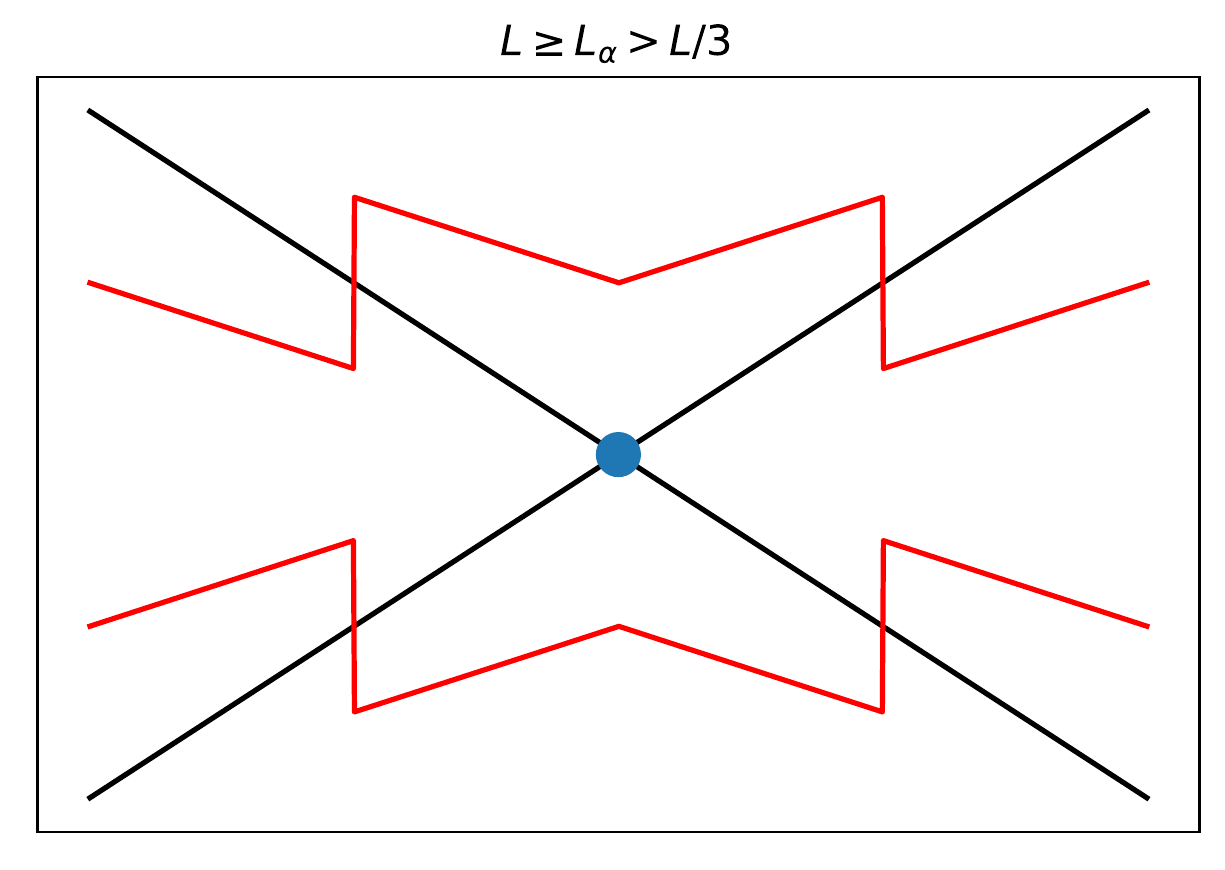}
\includegraphics[width=0.32\columnwidth]{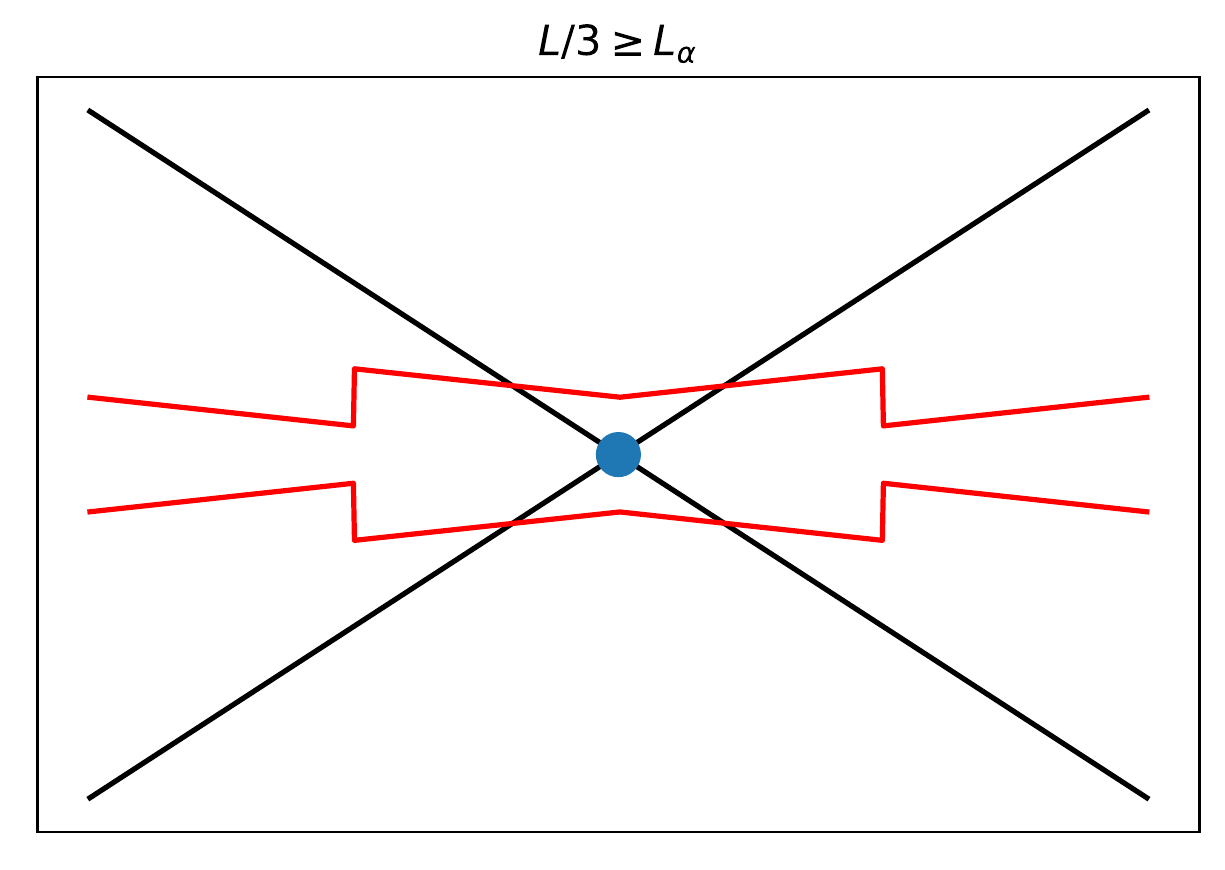}
\includegraphics[width=0.32\columnwidth]{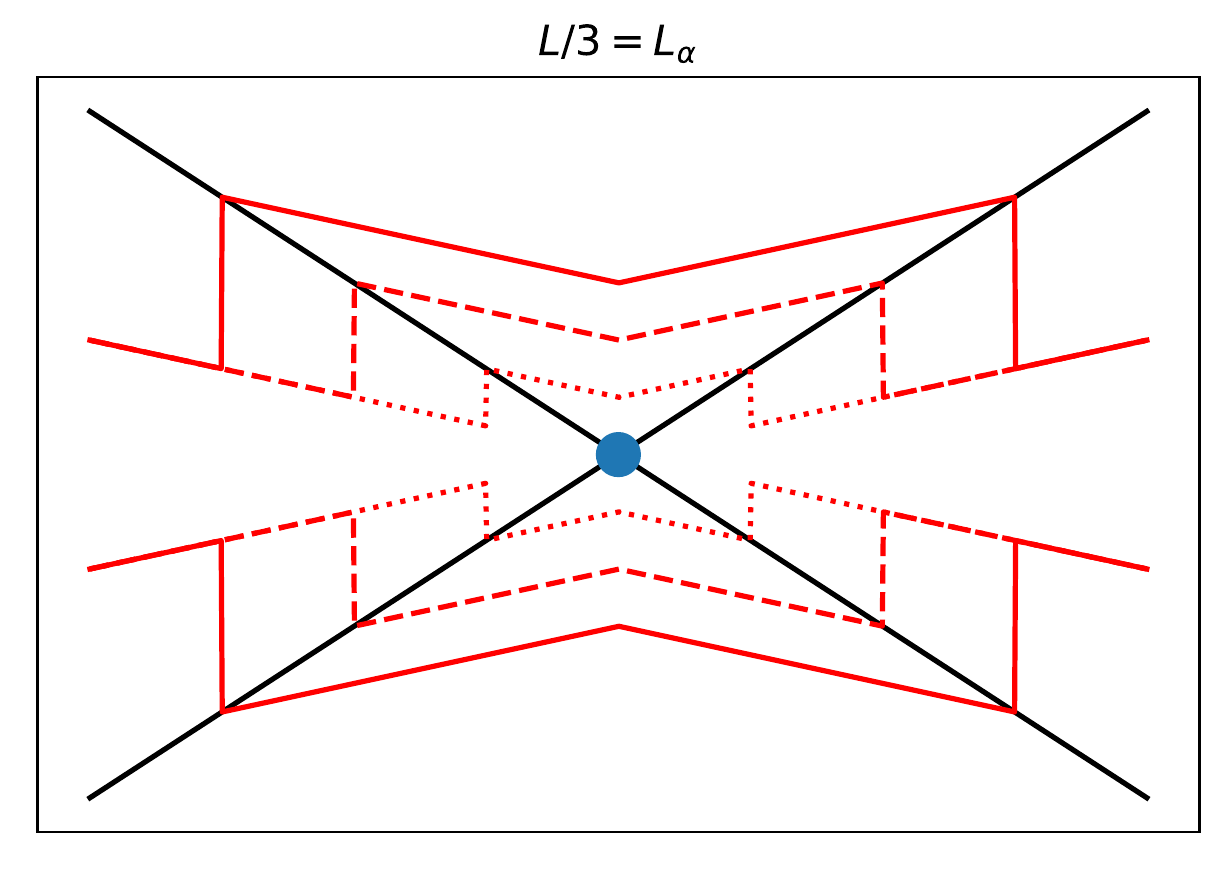}
\caption{Schematic of how $l_{\alpha}$, the maximum distance from a known point over which Lipschitz bounds are tighter than $L_{\alpha}$ bounds, as a function of the relation between $L$ and $L_{\alpha}$ (see also Theorem \ref{thm:compare_est_methods}).}
\label{fig:schematic_of_cases_L_vs_L_alpha}
\end{center}
\vskip -0.2in
\end{figure}

\paragraph{Theorem \ref{thm:compare_est_methods}}

Let $f(x)$ be defined over the metric space $\mathcal{X}$, and assume the value of $f(x)$ is known at $N$ points. Assume $f(x)$ is Lipschitz continuous with Lipschitz constant $L$. For any value of $\alpha$ and its corresponding $L_{\alpha}$ over $f(x)$, define $l_{\alpha}$ to be
\begin{align}
    l_{\alpha} = \begin{cases}
      \alpha & L \geq L_{\alpha} > L / 3 \\
      \frac{2 L_{\alpha} \alpha}{L - L_{\alpha}} & L / 3 \geq L_{\alpha}. \\
   \end{cases}
\end{align}
Then, if $\frac{C_D}{C_{D_{\mathcal{X}}}} \left( \frac{l_{\alpha}}{l_{\mathcal{X}}} \right)^D N < 1$, the fraction of volume of $\mathcal{X}$ over which the $L_{\alpha}$ bound is tighter than the Lipschitz bound is at least
\begin{align}
\label{eq:comparison_between_lip_and_l_alpha}
      1 - \frac{C_D}{C_{D_{\mathcal{X}}}} \left( \frac{l_{\alpha}}{l_{\mathcal{X}}} \right)^D N.
\end{align}

\begin{proof}
We will first prove that, if a point $x\in\mathcal{X}$ satisfies $d(x, x') > l_{\alpha}$ for every $x'$ at which $f(x')$ is known, then the $L_{\alpha}$-bound is tighter than the Lipschitz bound. We prove the result for the upper bound, but the proof is similar for the lower bound.

Consider the case in which the value of $f(x)$ is known at only one point, $x_0$. When $d(x_0, x) > \alpha$, the upper bounds $\hat{f}_{UB}^{\alpha}(x_0)$ and $\hat{f}_{UB}(x_0)$ are given by $d(x_0, x) L_{\alpha}$ and $d(x_0, x) L$, respectively, and because $L > L_{\alpha}$, we have that $\hat{f}_{UB}^{\alpha}(x_0) > \hat{f}_{UB}(x)$ when $d(x_0, x) > \alpha$, for any value of $L_{\alpha}$.

When $L / 3 \geq L_{\alpha}$, we can obtain an even smaller value for $l_{\alpha}$ by equating Equations~\ref{eq:lip_est} and~\ref{eq:lip_alpha_est} and solving for $d(x_0, x)$ to obtain $l_{\alpha} = \frac{2 L_{\alpha} \alpha}{L - L_{\alpha}}$. Figure \ref{fig:schematic_of_cases_L_vs_L_alpha} in the appendix shows a schematic demonstrating the different cases for $L \geq L_{\alpha} > L / 3$ and $L / 3 \geq L_{\alpha}$. Note that the different cases are determined only by comparing $L$ and $L_{\alpha}$, but are not affected by $\alpha$ itself, as demonstrated in Figure \ref{fig:schematic_of_cases_L_vs_L_alpha} (right).

Now assume the value of $f(x)$ is known at $N$ points $x_n$, such that for all points $d(x_n, x) > l_{\alpha}$. Denote  $\hat{f}_{UB}^{\alpha, x_i}(x)$ and $\hat{f}_{UB}^{x_i}(x)$ as the upper bound we would have computed for $f(x)$ if only $f(x_i)$ were known. Because $\hat{f}_{UB}^{\alpha, x_i}(x) > \hat{f}_{UB}^{x_i}(x)$ for all $i$, we also have that $\min_i \hat{f}_{UB}^{\alpha, x_i}(x) > \min_i \hat{f}_{UB}^{x_i}(x)$ and therefore $\hat{f}_{UB}^{\alpha}(x) > \hat{f}_{UB}(x)$.

A similar proof can be provided for the lower bound, and the two bounds together imply that Lipschitz estimation may only provide tighter bounds in balls of radius $l_{\alpha}$ around known points. 

The largest volume of the space over which Lipschitz estimation is tighter than $L_{\alpha}$ estimation is achieved if all balls of radius $l_{\alpha}$ centered at the $N$ points where $f(x)$ is known, and that volume will be $N C_D l_{\alpha}^D$, and the fraction that volume is of $\mathcal{X}$ is $\frac{V_{\mathcal{X}} - N C_D l_{\alpha}^D}{V_{\mathcal{X}}}$, and algebraic manipulation yields Equation~\ref{eq:comparison_between_lip_and_l_alpha}.
\end{proof}

\section{Proof of Proposition \ref{thm:max_lip_alpha}}
\label{appendix:proof_max_lip_alpha}

As part of the proof, we will start with the following Lemma:

\begin{lemma}
\label{thm:max_q_diff_any_points}
For an MDP satisfying the assumptions of Proposition~\ref{thm:max_lip_alpha}, for any $x, x' \in \mathcal{X}$,
\begin{align}
\label{eq:max_q_diff_any_points}
|Q^*(x) - Q^*(x')| \leq \Delta Q^*_{\max} \ceil*{\frac{d(x, x')}{d_{\min}}}.
\end{align}
\end{lemma}

\emph{Proof of Lemma \ref{thm:max_q_diff_any_points}.} Define $i_* \equiv \ceil*{\frac{d(x, x')}{d_{\min}}}$. Let $x_i = x + i (x' - x)$ for $i = 0, 1,...,i_* - 1$. Because $\mathcal{X}$ is convex, $x_i \in \mathcal{X}$ for all $i$. Because $d(x_i, x_{i+1}) = d_{\min}$, by Corollary~\ref{thm:max_jump_in_mdp_global}, for all  $i$ we have that $Q^*(x_{i+1}) \leq Q^*(x_{i}) + \Delta Q^*_{\max}$. Thus, since $x_0 = x$, we can write $Q^*(x_{i}) \leq Q^*(x) + i \Delta Q^*_{\max}$. Furthermore, because $d(x_{{i_*} - 1}, x') \leq d_{\min}$, we have one additional inequality: $Q^*(x') \leq Q^*(x_{i_* - 1}) + \Delta Q^*_{\max} \leq Q^*(x) + \Delta Q^*_{\max} i_*$. A similar inequality can be derived for the lower bound of $Q^*(x')$. \qed

\emph{Proof of Proposition \ref{thm:max_lip_alpha}.} From Equation \ref{eq:lip_alpha}, all we need to show is that for all $x, x' \in \mathcal{X}$ such that $d(x, x') \geq d_{\min}$, the inequality $\frac{|Q^*(x) - Q^*(x')|}{d(x, x')} \leq \frac{2 \Delta Q^*_{\max}}{d_{\min}}$ holds. By Lemma~\ref{thm:max_q_diff_any_points}:

\begin{align}
\label{eq:l_alpha_bound_proof_ineq}
    \frac{|Q^*(x) - Q^*(x')|}{d(x, x')} \leq \Delta Q^*_{\max} \ceil*{\frac{d(x, x')}{d_{\min}}} \frac{1}{d(x, x')}.
\end{align}

For $d(x, x') \geq d_{\min}$ , the function $\ceil*{\frac{d(x, x')}{d_{\min}}} \frac{1}{d(x, x')}$ obtains its supremum at $d(x, x') = \lim_{\epsilon \rightarrow 0} d_{\min} + \epsilon$, and the value of the supremum is $\frac{2}{d_{\min}}$. Substituting this value in Equation~\ref{eq:l_alpha_bound_proof_ineq} completes the proof. \qed

\section{Proof of Proposition \ref{thm:zooming_prop}}
\label{appendix:proof_zooming}

\paragraph{Proposition \ref{thm:zooming_prop}.}
In the ZoomRL algorithm introduced in \citet{touati2020zooming}, for any value of $\alpha$ such that $L_{\alpha} \leq L/2$, substituting $L_{\alpha}$-bounds for $L$ results in improved regret for at least the first $\frac{2}{3 \alpha}$ episodes.

\begin{proof}

The ZoomRL algorithm works by partitioning the state-action space into balls of decreasing radii, and constructing value estimates for each ball. Actions are selected optimistically by using an Lipschitz-based upper bound on the Q values of state-actions in each ball, called the \emph{index} of the ball. Therefore, to prove Proposition \ref{thm:zooming_prop}, we just need to prove that the indices obtained using $L_{\alpha}$ are all tighter than the ones obtained using $L$.

The proof will be divided into two parts. We will first show that for at least the first $\frac{2}{3 \alpha}$ episodes, the radius of all balls is greater or equal to $3 \alpha$ (part I). In the second part we will show that this lower bound on the radii of all balls, together with the assumption that $L_{\alpha} \leq L/2 $, implies that the $L_{\alpha}$ bounds in the expression for the indices of balls used by the ZoomRL algorithm results in tighter bounds (part II).

(I) For each step $h$, and a particular episode number $k$, let $\rad(B_{h}^{k})$ denote the radius of the corresponding ball. Each ball has a radius of half of its parent, and a new ball is created only when its parent has been visited at least $\frac{1}{\rad(B_{h}^{k})}$ times . When the learner creates new balls as quickly as possible, which happens when each ball is the only child of its parent, to generate a ball of size smaller than $3 \alpha$ we would create balls of radii $(\frac{1}{2}, \frac{1}{2}, ... \frac{1}{2^{n-1}}, \frac{1}{2^{n}})$, where $n$ is chosen such that 

\begin{align}
\label{eq:zoomrl_constraint_on_k}
\frac{1}{2^{n-1}} \geq 3 \alpha > \frac{1}{2^{n}}.
\end{align}

The sum of episodes needed to generate such a ball is $K = 1 + 2 + 4 + 2^{n+1} = 2^n - 1 \approx 2^n$. Thus as long as $K < 2^n$, the smallest ball's radius is larger than $3 \alpha$. By Equation \ref{eq:zoomrl_constraint_on_k}, the condition $K < 2^n$ can be written as

$$ K < 2^n \leq \frac{2}{3 \alpha}. $$

(II) We now show that if $L_{\alpha} \leq L/2 $ and $3 \alpha \leq \rad(B_{h}^{k})$ for all balls, then substituting the $L_{\alpha}$ bounds in the expression for the indices of balls used by the ZoomRL algorithm results in tighter bounds.

The expression for the indices in the ZoomRL algorithm is

\begin{align}
    \idx(B) &\equiv L \cdot \rad(B) \ + \min_{B' \ | \ \rad(B') \geq  \rad(B)} \left[  \hat{Q}(B') + L \cdot \dist(B, B') \right]
\end{align}

Where $\hat{Q}(B')$ denotes the value estimate for the center of a ball, and $\dist(B, B')$ the distance between the centers of the balls. For simplicity, we drop the super- and sub- scripts from the $B_{h}^{k}$ notation for the balls.

The index of a ball is an upper bound for the value of all state-action pairs contained within it. We wish to write an analogous bound obtained using $L_{\alpha}$. We first address the alternative for the $L \cdot \rad(B)$ term. We must make sure that the loosely bounded region at a distance less than $\alpha$ from the center of the ball cannot have violation of the bound. However, because $3 \alpha \leq \rad(B)$, we know that $L_{\alpha} \cdot \rad(B) \geq 3 \alpha L_{\alpha}$, and the $L_{\alpha} \cdot \rad(B)$ properly upper bounds the difference between the upper bound on the Q-value at the center of ball $B$ and the Q-value at all other points in the ball.

Similarly, to replace $L$ with $L_{\alpha}$ in the $L \cdot \dist(B, B')$ term, we must make sure that we do not get bound violations when $\dist(B, B') < \alpha$. To do this we will split our analysis to two parts. First, we assume that $\dist(B, B') \geq \alpha$, and then we can safely replace $L$ with $L_{\alpha}$. In that case, we can see that

\begin{align}
    \idx_{L_{\alpha}}(B) &\equiv L_{\alpha} \cdot \rad(B) \ + \min_{B' \ | \ \rad(B') \geq  \rad(B)} \left[  \hat{Q}(B') + L_{\alpha} \cdot \dist(B, B') \right]
\end{align}

is a valid index (upper bound) for ball $B$. Then, because $L_{\alpha} \leq L$, the inequality

\begin{align}
    L_{\alpha} \cdot \rad(B)  +  \left[  \hat{Q}(B') + L_{\alpha} \cdot \dist(B, B') \right] \\ \nonumber
    \leq L \cdot \rad(B)  +  \left[  \hat{Q}(B') + L \cdot \dist(B, B') \right]
\end{align}
holds for all $B'$, proving that $\idx_{L_{\alpha}}(B) \leq \idx_{L}(B)$, or, in other words, the $L_{\alpha}$ index is tighter than the index obtained using using $L$.

In the case that $\dist(B, B') < \alpha$, we can upper bound the contribution to the bound of the $L_{\alpha} \cdot \dist(B, B')$ by $3 \alpha L_{\alpha}$. Using the upper bound to the term in the index, makes it a valid upper bound. Now all that remains is to demand that it is always tighter than the Lipschitz based bound. In other words, we wish to find parameters for which

\begin{align}
    & L_{\alpha} \cdot \rad(B)  +  \left[  \hat{Q}(B') + L_{\alpha} \cdot \dist(B, B') \right] \\ \nonumber
    & \leq L_{\alpha} \cdot \rad(B)  +  \left[  \hat{Q}(B') + 3 \alpha L_{\alpha} \right] \\ \nonumber
    & \leq L \cdot \rad(B)  +  \left[  \hat{Q}(B') + L \cdot \dist(B, B') \right]
\end{align}

for all $B'$. To do this we will ensure the inequality between the second and third lines hold even when we set $\dist(B, B') = 0$. Eliminating $\hat{Q}(B')$ from both sides leaves us with showing that

\begin{eqnarray}
    L_{\alpha} \cdot \rad(B)  + 3 \alpha L_{\alpha}\leq L \cdot \rad(B) \\ \nonumber
    \Rightarrow \frac{L_{\alpha}}{L} \leq \frac{\rad(B)}{\rad(B) + 3 \alpha}.
\end{eqnarray}

Because $\rad(B) \geq 3 \alpha$, we have that

\begin{align}
    \frac{1}{2} = \frac{3 \alpha }{3 \alpha  + 3 \alpha } \leq \frac{\rad(B)}{\rad(B) + 3 \alpha } \
\end{align}
And since we have assumed $L_{\alpha} \leq L /2$, this completes the proof.
    
\end{proof}

\end{document}